\documentclass[twoside,11pt]{article}

\usepackage{blindtext}
\usepackage{amsmath}
\usepackage{jmlr2e}

\usepackage{lastpage}
\jmlrheading{XX}{2024}{1-\pageref{LastPage}}{X/XX; Revised X/XX}{X/XX}{21-0000}{Xudong Zheng and Yuecai Han}

\ShortHeadings{Predicting path-dependent processes by deep learning}{Zheng and Han}
\firstpageno{1}

\begin{document}

\title{Predicting path-dependent processes by deep learning}

\author{\name Xudong Zheng \email zxd22@mails.jlu.edu.cn \\
       \addr School of Mathematics\\
       Jilin University\\
       Changchun, Qianjin Street 2699, CHN
       \AND
       \name Yuecai Han \email hanyc@jlu.edu.cn \\
       \addr School of Mathematics\\
       Jilin University\\
       Changchun, Qianjin Street 2699, CHN}

\editor{My editor}

\maketitle

\begin{abstract}%   <- trailing '%' for backward compatibility of .sty file
In this paper, we investigate a deep learning method for predicting path-dependent processes based on discretely observed historical information. This method is implemented by considering the prediction as a nonparametric regression and obtaining the regression function through simulated samples and deep neural networks. When applying this method to fractional Brownian motion and the solutions of some stochastic differential equations driven by it, we theoretically proved that the $L_2$ errors converge to 0, and we further discussed the scope of the method. With the frequency of discrete observations tending to infinity, the predictions based on discrete observations converge to the predictions based on continuous observations, which implies that we can make approximations by the method. We apply the method to the fractional Brownian motion and the fractional Ornstein-Uhlenbeck process as examples. Comparing the results with the theoretical optimal predictions and taking the mean square error as a measure, the numerical simulations demonstrate that the method can generate accurate results. We also analyze the impact of factors such as prediction period, Hurst index, etc. on the accuracy.
\end{abstract}


\section{Introduction}
Stochastic processes are often used in probability applications to model a system. For example, analyzing stochastic processes in intricate systems is the goal of modeling intracellular transport, and modeling asset prices by stochastic processes is a tool of financial mathematics. Markovian stochastic processes are widely used in these fields. However, recent studies show that some actual data exhibit path dependence, which means that the behavior of a real process at a given time $t$ depends on the conditions at the point in time as well as the whole history up to that moment. An increasing amount of evidence suggests that many cellular systems show anomalous diffusion rather than classical Brownian motion, as a result of single-particle tracking experiments that directly see molecular movement (see, e.g., \cite{Caccetta2011}). In financial mathematics, recent studies have demonstrated that rough stochastic volatility-type dynamics may better describe spot volatilities, and it can be confirmed that volatility is rough by price data of at-the-money options on the S\&P500 index with short maturity (see, e.g., \cite{Livieri2018}).

Our aim in this paper is to propose a framework for predicting path-dependent stochastic processes $\left(W_t, \ t\geq0\right)$ by discrete historical information and deep neural networks, i.e., to solve
\begin{displaymath}
    \mathbb{E}\left[W_T \mid \mathcal{F} ^{D,V,N}_s\right],
\end{displaymath}
where, $\mathcal{F} ^{D,V,N}_s :=\sigma \overline{\left\{V_{t_1}, \dots, V_{t_N} \right\}} \ \text { for } 0  < t_1 < \dots < t_N= s < T$ and there are connections between the stochastic processes $\left(V_t, \ t\geq0\right)$ and $\left(W_t, \ t\geq0\right)$. Naturally, the case $V_t=W_t$ is also under our consideration. Here, "$D$" means that the historical information is "discrete", and later in the paper we use "$C$" to denote that the historical information is "continuous". In this paper, we mainly focus on the predictions of the fractional Brownian motion (one of the widely studied path-dependent stochastic processes) and some stochastic processes driven by it. Without loss of generality, let $t_i-t_{i-1}=s/N$ in the following.

 The fractional Brownian motion (fBm) $\left(B_{t}^{H}, \  t\geq0\right)$ with Hurst index $H \in \left(0,1\right)$ is a self-similar process, i.e. $ \left ( B_{at}^{H}, \ t \geq 0 \right)  $ and $  ( a^{H} B_{t}^{H}, \ t \geq 0  ) $ have the same probability law for all $a>0$. Compared to the standard Brownian motion (sBm), it displays long-range dependence and positive correlation when $H \in \left(1/2,1\right)$ and negative correlation when $H \in \left(0,1/2 \right)$. This special property makes it a useful tool for modeling a wide range of phenomena, such as financial time series, turbulence, and climate data (see, e.g., \cite{Mandelbrot1968}). \cite{Kepten2011} discovered strong evidence that the telomeres motion obeys fractional dynamics, and the ergodic dynamics are seen empirically to fit the fBm. In finance, \cite{Gatheral2018} showed that log-volatility behaves as the fBm at any reasonable timescale by estimating volatility from high-frequency data. 

Compared to the sBm, the fBm with $H \ne 1/2$ is neither a Markov process nor a semi-martingale, therefore some classical stochastic analysis tools are not applicable. Theoretical research on it has been active in recent years. Advances have also been made in the study of its stochastic fractional calculus and statistical theory, which provides a framework for modeling stochastic processes with it (see, e.g., \cite{NUALART2009391, Wang1996}). \cite{Decreusefond1999} used the stochastic calculus of variations to develop stochastic analysis theory for the functionals of fBms, while \cite{Duncan2000} defined the multiple and iterated integrals of the fBm and provided various properties of these integrals. \cite{Elliott2003} presented an extended framework, enabling processes with all indexes to be considered under the same probability measure. \cite{Biagini2004} introduced the theory of stochastic integration for the fBm based on white-noise theory and differentiation.

Statistical inference for diffusion processes that satisfy stochastic differential equations driven by Brownian motions has been studied previously, with \cite{Rao2010} providing a comprehensive overview of the various approaches. Attention has been given to studying similar problems for stochastic processes driven by the fBm. \cite{LEBRETON1998263} studied parameter estimation and filtering for simple linear models driven by the fBm. \cite{Kleptsyna2002} investigated the parameter estimation problems for the fractional Ornstein-Uhlenbeck-type process. \cite{Tudor2007} studied the maximum likelihood estimator for the drift parameter of stochastic processes satisfying stochastic equations driven by the fBm. It was further demonstrated that the existence and strong consistency of the maximum likelihood estimation for both linear and nonlinear equations. The problem of parametric and nonparametric estimation of processes driven by the fBm was comprehensively discussed by \cite{Rao2010}. \cite{Dang2017} estimated both the Hurst index and the stability indices of a $H$-self-similar stable process, and obtained consistent estimators with convergence rate. \cite{Kubilius2017} gave an overview of the Hurst index estimators for the fBm.

The subject of how to forecast the future based on observable data and its potential accuracy naturally emerges once a model driven by the fBm has been established. In other words, when using the fBm to simulate path-dependent stochastic processes, explicit predictor formulas are useful tools. This issue is related to the flow property and prediction theory. \cite{Nualart2006} explained that there exists a sBm whose filtration is the same as the one of the fBm, since the existence of a bijective transfer operator to express the sBm as a Wiener integral with respect to the fBm. But since all of the historical data will play a role, its non-Markovianity makes prediction more difficult.

When the observed information is continuous, representing the predictors of fBm as integrals over the process's observed portion makes sense. \cite{gripenberg_norros_1996} investigated the prediction weight function by solving a weakly singular integral equation and calculated the variance of the predictor. Additionally, they provided a characterization of the conditional distribution of the future given the past. The result contains double integrals and requires continuous filtrations, which is difficult to implement. Explicit expressions for some conditional expectations were provided by \cite{DUNCAN2006128} for the prediction of some stochastic processes constructed as solutions of stochastic differential equations with the fBm. \cite{Fink2013} constructed the conditional distributions of the fBm and related processes, including the fractional Ornstein-Uhlenbeck (fOU) processes and the fractional Cox-Ingersoll-Ross (fCIR) processes, using a prediction formula for the conditional expectation, Fourier techniques and Gaussianity.

When the observed information is discrete, however, the situation is different because the filtrations generated by the sBm and the fBm are not equivalent anymore (see, \cite{Nualart2006}). In this case, the prediction problem can be transformed into a nonparametric regression problem. More specifically, the aim is to construct an estimator called regression function such that the mean square error (MSE) is minimized, and the regression function that minimizes the MSE is the conditional expectation (see, e.g., \cite{Gyrfi2002ADT}). For the case where the filtrations are composed of discrete sBms, the polygonal approximation of the fBm is the best approximation in the MSE sense (see, \cite{Hu2005IntegralTA}). The optimal predictor for long-memory time series, including the fractional auto-regressive integrated moving average model, has been well studied (see, \cite{Bhansali2003PredictionOL}). An asymptotic linear minimum-mean-square-error predictor for the discrete-time fBm was proposed by \cite{Yao2003PredictionOL}. These results also show that due to the long-range dependence, when the number of discrete observations increases, the prediction of the future (or the regression function) becomes quite complex because of the increased dimensionality of the variables.

The use of neural networks has yielded impressive outcomes for high-dimensional problems across various fields (see, e.g., \cite{Shrestha2019}). Deep learning, which uses multi-layer neural networks with many hidden layers, is particularly well-suited for high-dimensional problems and real-world data, as noted by \cite{Schmidhuber2015}. Deep learning performs well in solving fractional partial differential equations  (see, \cite{Lu2021}), classifying single-particle trajectories as fBms and estimating the Hurst index (see, \cite{Granik2019}), and solving the problem of optimally stopping a fBm (see, \cite{Becker2019}). In nonparametric regression, it was demonstrated that least squares estimates based on multilayer feedforward neural networks can avoid the curse of dimensionality, provided that a smoothness condition and an appropriate limit on the structure of the regression function exist (see, \cite{Bauer2019, Langer2021, Kohler2021}).

The primary goal of this paper is to propose a deep learning method that can predict path-dependent stochastic processes. The paper is organized as follows. We briefly recall the properties of the fBm, nonparametric regression, and deep neural networks in Section \ref{sec:2}. The main results are in Section \ref{sec:3}. We theoretically demonstrate the applicability of the method to the fBm and the solutions of some stochastic differential equations driven by it, and further discuss the scope of the method. We also show that the outputs converge to the predictions under continuous observations as the frequency of discrete observations tends to infinity. In Section \ref{sec:4}, we apply the method to two examples: the fBm and the fOU process. We compare the results with the theoretical optimal predictions, and taking the mean square error as a measure, the numerical simulations demonstrate that the method can generate accurate results. We also compare the impact of factors such as prediction period, Hurst index, etc. on the accuracy. We conclude the paper in Section \ref{sec:5}.

\section{Fractional Brownian motion, nonparametric regression and neural networks}
\label{sec:2}
\subsection{Fractional Brownian motion}

A fBm $\left(B_t^H, \ t \geq 0\right)$ with Hurst index $H \in(0,1)$ is a centered Gaussian process with covariance 
\begin{displaymath}
    \mathbb{E}\left(B_t^H B_s^H\right)=\frac{1}{2}\left(t^{2 H}+s^{2 H}-|t-s|^{2 H}\right), \quad s, t \geq 0 .
\end{displaymath}
When $H=\frac{1}{2}$, it becomes a sBm and we denote $B_t^{\frac{1}{2}}$ by $B_t$. We can use the following moving-average stochastic integral representation of the fBm, which is used in this paper (see, \cite{Hu2005IntegralTA}):
\begin{equation}
    \label{ker}
    B_t^H=\int_0^t K_H(t, s) \mathrm{d} B_s, \quad 0 \leq t<\infty,
\end{equation}
where
\begin{displaymath}
   K_H(t, s)=\kappa_H\left[\left(\frac{t}{s}\right)^{H-\frac{1}{2}}(t-s)^{H-\frac{1}{2}}-\left(H-\frac{1}{2}\right) s^{\frac{1}{2}-H} \int_s^t u^{H-\frac{3}{2}}(u-s)^{H-\frac{1}{2}} \mathrm{d} u\right],
\end{displaymath}
\begin{displaymath}
    \kappa_H=\sqrt{\frac{2 H \Gamma\left(\frac{3}{2}-H\right)}{\Gamma\left(H+\frac{1}{2}\right) \Gamma(2-2 H)}} .
\end{displaymath}
If $H>\frac{1}{2}$, then $Z_H(t, s)$ can be written as
\begin{displaymath}
    K_H(t, s)=\left(H-\frac{1}{2}\right) \kappa_H s^{\frac{1}{2}-H} \int_s^t u^{H-\frac{1}{2}}(u-s)^{H-\frac{3}{2}} \mathrm{d} u.
\end{displaymath}
Moreover, there exists a bijective transfer operator that expresses $B_t$ as a integral with respect to $B_t^H$. Define $\mathcal{F} ^{C,B^H}_s :=\sigma \overline{\left\{B_v^H, \ v \in[0, s]\right\}}$, $\mathcal{F} ^{C,B}_s :=\sigma \overline{\left\{B_v, \ v \in[0, s]\right\}}$. The filtration of $B_{t}^{H}$ and $B_{t}$ are equivalent, i.e. $\mathcal{F} ^{C,B^H}_s=\mathcal{F} ^{C,B}_s$. In this paper, we focus on a given time interval $[0, T]$ with $T>0$.

\subsection{Nonparametric regression}
In nonparametric regression analysis, we consider a random vector $X \in \mathbb{R}^d$ and a random variable $Y \in \mathbb{R}$ satisfying $\mathbb{E} \left[Y^2\right]<\infty$. Nonparametric regression is aimed at predicting the value of the response variable $Y$ from the value of the observation vector $X$ by constructing an optimal predictor $m^*: \mathbb{R}^d \rightarrow \mathbb{R}$. It is usual that the aim is to minimize the MSE, i.e.,
\begin{displaymath}
\mathbb{E}\left[\left|Y-m^*(X)\right|^2\right]=\min _{f: \mathbb{R}^d \rightarrow \mathbb{R}} \mathbb{E}\left[|Y-f(X)|^2\right].
\end{displaymath}
Let the regression function $m\left(x\right)=\mathbb{E}\left[Y \mid X=x\right]$, and it is the optimal predictor since $m$ satisfies
\begin{displaymath}
    \mathbb{E}\left[|Y-f(X)|^2\right]=\mathbb{E}\left[|Y-m(X)|^2\right]+\int|f(x)-m(x)|^2 \mathbf{P}_X(\mathrm{d} x).
\end{displaymath}
For the application, we are able to obtain a data set $\mathcal{D}_n=\left\{\left(X_1, Y_1\right), \ldots,\left(X_n, Y_n\right)\right\}$, but do not know the distribution of $(X, Y)$, where $(X, Y),\left(X_1, Y_1\right), \ldots,\left(X_n, Y_n\right)$ are independent and identically distributed. The target of nonparametric regression is to construct regression estimates $m_n(\cdot)=m_n\left(\cdot, \mathcal{D}_n\right)$ such that the $L_2$ errors $\int\left|m_n(x)-m(x)\right|^2 \mathbf{P}_X(\mathrm{d} x)$ are as small as possible.

Compared to the parametric estimation, which assumes that the fixed structure of the regression function depends only on a finite number of parameters, nonparametric estimation methods do not assume that the regression function can be characterised by a finite number of parameters, but rather estimate the entire function from the data. The class of regression functions must be restricted in order to obtain theoretical results such as the rate of convergence (see, \cite{Gyrfi2002ADT}). In the following we introduce the definitions of $\left(p, C\right)$-smooth and hierarchical composition models (see, e.g., \cite{Kohler2021}) which are used to define the classes of regression functions discussed in this paper.
\begin{definition}
Let $p=q+s$ for some $q \in \mathbb{N}_0$ and $0<s \leq 1$. A function $m: \mathbb{R}^d \rightarrow \mathbb{R}$ is called $\left(p, C\right)$-smooth, if for every $\alpha=\left(\alpha_1, \ldots, \alpha_d\right) \in \mathbb{N}_0^d$ with $\sum_{j=1}^d \alpha_j=q$, the partial derivative $\partial^q m /\left(\partial x_1^{\alpha_1} \ldots \partial x_d^{\alpha_d}\right)$ exists and satisfies
\begin{displaymath}
   \left|\frac{\partial^q m}{\partial x_1^{\alpha_1} \ldots \partial x_d^{\alpha_d}}\left(x\right)-\frac{\partial^q m}{\partial x_1^{\alpha_1} \ldots \partial x_d^{\alpha_d}}\left(z\right)\right| \leq C\|x-z\|^s ,
\end{displaymath}
for all $x, z \in \mathbb{R}^d$, where $\|\cdot\|$ denotes the Euclidean norm.
\end{definition}
\begin{definition}
Let $d \in \mathbb{N}$, $m: \mathbb{R}^d \rightarrow \mathbb{R}$ and let $\mathcal{P}$ be a subset of $(0, \infty) \times \mathbb{N}$.\\
(1) We say that $m$ satisfies a hierarchical composition model of level 0 with order and smoothness constraint $\mathcal{P}$, if there exists $K \in\{1, \ldots, d\}$ such that $m(x)=x_{K}$, for all $x=\left(x_{1}, \ldots, x_{d}\right) \in \mathbb{R}^d$.\\
(2) We say that $m$ satisfies a hierarchical composition model of level $l+1$ with order and smoothness constraint $\mathcal{P}$, if there exist $(p, K) \in \mathcal{P}$, $C>0$, $g: \mathbb{R}^K \rightarrow \mathbb{R}$ and $f_1, \ldots, f_K:\mathbb{R}^d \rightarrow \mathbb{R}$, such that $g$ is $(p, C)$-smooth, $f_1, \ldots, f_K$ satisfy a hierarchical composition model of level $l$ with order and smoothness constraint $\mathcal{P}$ and $m(x)=g\left(f_1(x), \ldots, f_K(x)\right)$, for all $x \in \mathbb{R}^d$.
\end{definition}
\begin{definition}
\label{defhier}
For $l=1$ and some order and smoothness constraint $\mathcal{P} \subseteq(0, \infty) \times \mathbb{N}$, the space of hierarchical composition models becomes
\begin{displaymath}
\begin{aligned}
\mathcal{H}\left(1, \mathcal{P}\right)=  \Big\{ & h: \mathbb{R}^d \rightarrow \mathbb{R}: h\left(x\right)=g\left(x_{\pi\left(1\right)}, \ldots, x_{\pi\left(K\right)}\right), \text { where } \Big.\\
& g: \mathbb{R}^K \rightarrow \mathbb{R} \text { is }(p, C) \text {-smooth for some }(p, K) \in \mathcal{P}, \\
& \Big.C>0 \text { and } \pi:\{1, \ldots, K\} \rightarrow\{1, \ldots, d\}\Big\} .
\end{aligned}    
\end{displaymath}
For $l>1$, we recursively define
\begin{displaymath}
\begin{aligned}
\mathcal{H}(l, \mathcal{P}):=  \Big\{& h: \mathbb{R}^d \rightarrow \mathbb{R}: h(x)=g\left(f_1(x), \ldots, f_K(x)\right), \text { where } \Big. \\
& g: \mathbb{R}^K \rightarrow \mathbb{R} \text { is }(p, C) \text {-smooth for some }(p, K) \in \mathcal{P}, \\
& \Big.C>0 \text { and } f_i \in \mathcal{H}(l-1, \mathcal{P})\Big\} .
\end{aligned}    
\end{displaymath}
\end{definition}
\begin{remark}
\label{remark4}
Regression functions in $\mathcal{H}(l, \mathcal{P})$ are able to characterize some input-output relationships in actuality. A reason for using the function class $\mathcal{H}(l, \mathcal{P})$ is that it describes a complex relationship that can be constructed as a recursive construction by modules. In other words, each module can describe a complex input-output relationship.
\end{remark}
\begin{remark}
\label{remark1}
Another reason for using $\mathcal{H}(l, \mathcal{P})$ in this paper is to overcome the curse of dimensionality. An optimal predictor can also be obtained with the $(p, C)$-smooth functions as the target of nonparametric regression. Although it is optimal, the optimal minimax rate of convergence in nonparametric regression for $(p, C)$-smooth functions is $n^{-\frac{2 p}{2 p+d}}$ (see, \cite{Stone1982}). We intend to consider the case where the number of discrete observations is large in the prediction problem. If $d$ is relatively large compared with $p$, the convergence rate becomes significantly slower, which is often referred to as the "curse of dimensionality". The only possible way to avoid it is to restrict the underlying function class. 

\cite{Kohler2021} showed that if the regression function satisfies a hierarchical composition model with $\mathcal{P}$(smoothness and order constraint), the convergence rate of the $L_2$ errors of least squares neural network regression estimates based on fully connected neural networks with ReLU activation functions doesn't depend on $d$, avoiding the curse of dimensionality.
\end{remark}

\subsection{Neural network regression estimator}
The architecture of a fully connected neural network can be represented by $(L, \lambda)$, where $L$ is the number of the hidden layer, often called the depth of the neural network, and $\lambda$ is the width vector $\lambda=\left(\lambda_{1}, \ldots, \lambda_{L}\right) \in \mathbb{N}^{L}$, which describes the number of neurons in each hidden layer. The neural network with the network architecture $(L, \lambda)$ can be represented as a function of the following form
\begin{equation}
\label{network}
  y^{\theta}\left(x\right): \mathbb{R}^{\lambda_{0}} \rightarrow \mathbb{R}^{\lambda_{L}}, \ x \mapsto \Gamma_{L}^{\theta} \sigma\left(\Gamma_{L-1}^{\theta} \ldots \sigma\left(\Gamma _{2}^{\theta} \sigma\left(\Gamma_{1}^{\theta} x\right)\right)\right),  
\end{equation}
where $\Gamma^{\theta}_{l}$ is the $\lambda_{l+1} \times\left(\lambda_{l}+1\right)$ matrix representing the parameters of the $l$th layer, and $\sigma$ is a deterministic nonlinear function, which is called the activiation function. The components of the parameter $\theta \in \mathbb{R}^q$ of $y^\theta$ consist of the entries of the matrices $A_1 \in \mathbb{R}^{\lambda_1 \times d}, \ldots, A_{L-1} \in \mathbb{R}^{\lambda_{L-1} \times \lambda_{L-2}}, A_{L}\in \mathbb{R}^{\lambda_{L} \times \lambda_{L-1}}$ and the vectors $b_1 \in \mathbb{R}^{\lambda_1}, \ldots, b_{L-1} \in$ $\mathbb{R}^{\lambda_{L-1}}, b_{L} \in \mathbb{R}^{\lambda_{L}}$ given by the affine functions
\begin{displaymath}
    \Gamma_i^\theta(x)=A_i x+b_i, \quad i=1, \ldots, L.
\end{displaymath}
In this paper, we choose the ReLU activation function 
\begin{displaymath}
    \sigma(s)=\left(\max \left\{s_{1}, 0\right\}, \ldots, \max \left\{s_{d}, 0\right\}\right)^{\top},
\end{displaymath}
where $s_{i}$ is the $i$th element of the $d$-dimensional vector $s$. Given $L$ hidden layers and $r$ neurons per layer, 
\begin{equation}
\label{ftheta}
    \mathcal{Y} \left(L, r\right):=\left \{ y^{\theta}:y^{\theta} \text{ is of the form (\ref{network}) with } \lambda_1=\lambda_2=\ldots=\lambda_{L-1}=r, \  \lambda_{L}=1 \right \}
\end{equation}
defines the space of neural networks used in the following. Then we define the neural network regression estimator as the minimizer of the empirical $L_2$-risk in the space of neural networks $\mathcal{Y} \left(L, r\right)$, that is, the estimator is defined by
\begin{equation}
\label{min}
    \tilde{m}_n(\cdot):=\arg \min _{y^{\theta} \in \mathcal{Y} \left(L, r\right)} \frac{1}{n} \sum_{i=1}^n\left|y^{\theta}\left(X_i\right)-Y_i\right|^2 .
\end{equation}
In numerical simulations, we use a gradient-based optimization method to solve the above optimization problem to get $\tilde{m}_n(\cdot)$.
\begin{remark}
    We use gradient-based methods in this paper since empirical studies show that they have great generalization performance (see, e.g., \cite{Caccetta2011}). For simplicity, it is assumed that the minimum in (\ref{min}) is existing and available. When this is not the case, our theoretical results also hold for the estimator $\tilde{m}_n(\cdot)$, with a small additional term. More details about the gradient-based methods are illustrated by \cite{Bottou2018}.
\end{remark}

\section{Main results}
\label{sec:3}
\subsection{Predictions of discrete-time fractional Brownian motion}

For the predictions and discrete approximation, we need the following lemma, which proves that deep neural networks can be used to predict Gaussian processes. For $d \in \mathbb{N}$, let $\mathbb{S}^{d \times d}$ denote the space of all positive semidefinite symmetric matrices of dimension $d$, and $T_\beta u :=\max \{-\beta, \min \{\beta, u\}\}$.

\begin{lemma}
\label{lemma0}
    Let random vectors $(X, Y),\left(X_1, Y_1\right), \ldots,\left(X_n, Y_n\right)$ be Gaussian, independent and identically distributed. Let $m^0\left(x\right)=\mathbb{E}\left[Y \mid X=x\right]$, $m^0_n=T_{c \log (n)} \tilde{m}^0_n$, where
    \begin{displaymath}
    \tilde{m}^0_n(\cdot):=\arg \min _{y^{\theta} \in \mathcal{Y} \left(L_n, r_n\right)} \frac{1}{n} \sum_{i=1}^n\left|y^{\theta}\left(X_i\right)-Y_i\right|^2 ,
    \end{displaymath}
    \begin{displaymath}
    L_n=\left\lceil c_1 \cdot \max _{(p, K) \in \mathcal{P}} n^{\frac{K}{2(2 p+K)}} \cdot \log n\right\rceil, \  r_n=r=\left\lceil c_2\right\rceil .   
    \end{displaymath}
    for some $c,\ c_1, \ c_2>0$ sufficiently large.
    Then 
    \begin{displaymath}
        \lim _{n \to \infty} \mathbb{E} \int\left|m^0_n(x)-m^0(x)\right|^2 \mathbf{P}_{X}(\mathrm{d} x)=0.
    \end{displaymath}     
\end{lemma}
\begin{proof}
$\left(X, Y\right)$ is multivariate normally distributed, and let $(Y, X)^{\top}\sim N(\mu, \Sigma)$ with mean $\mu \in \mathbb{R}^d$ and variance-covariance matrix $\Sigma \in \mathbb{S}^{d \times d}$. Partition
\begin{displaymath}
\mu=\left(\begin{array}{l}
\mu_1 \\
\mu_2
\end{array}\right) \quad \text { and } \quad \Sigma=\left(\begin{array}{ll}
\Sigma_{11} & \Sigma_{12} \\
\Sigma_{21} & \Sigma_{22}
\end{array}\right)    
\end{displaymath}
with $\mu_1 \in \mathbb{R}^k$, $\mu_2 \in \mathbb{R}^{d-k}$, $\Sigma_{11} \in \mathbb{S}^{k \times k}$, $ \Sigma_{22} \in \mathbb{S}^{(d-k) \times(d-k)}$, and $\Sigma_{12}^{\top}=\Sigma_{21} \in \mathbb{R}^{(d-k) \times k}$. We can obtain an explicit formulation by the property of the multivariate normal distribution: 
\begin{displaymath}
    m^0\left(x\right)=\mu_1+\Sigma_{12} \left(\Sigma_{22}\right)^{-1}\left(x-\mu_2\right)^{\top}.
\end{displaymath}
From the definition of hierarchical composition models $\mathcal{H}(l, \mathcal{P})$, it is clear that the regression function $m^0\left(\cdot\right) \in \mathcal{H}(l, \mathcal{P})$. By Theorem 1 of \cite{Kohler2021}, 
\begin{displaymath}
\mathbb{E} \int\left|m^0_n(x)-m^0(x)\right|^2 \mathbf{P}_{X}(d x) \leq c_3 \cdot(\log n)^6 \cdot \max _{(p, K) \in \mathcal{P}} n^{-\frac{2 p}{2 p+K}}    
\end{displaymath}
holds for some $c_3$ and sufficiently large $n$, $c$, $c_1$, $c_2$. We complete the proof.
\end{proof}
As we demonstrate later, Lemma \ref{lemma0} is the basis of all the results in this paper. Because we mainly discuss the prediction of path-dependent processes related to the fBm. Thus, the multivariate normal distribution is part of the connections between $\left(V_t, \ t\geq0\right)$ and $\left(W_t, \ t\geq0\right)$ stated in the introduction.
\begin{remark}
    Notice that the regression function $m^0\left(\cdot\right)$ is a $(p, C)$-smooth function, so we can use some classical methods to derive the regression function directly, instead of defining a hierarchical composition model. However, as stated in Remark \ref{remark1}, if a $(p, C)$-smooth function is used directly as the target of nonparametric regression, the optimal rate of convergence achieved is $n^{-\frac{2 p}{2 p+d}}$. The least squares neural network based on fully connected neural networks achieves a convergence rate $n^{-\frac{2 p}{2 p+K}}$ up to a logarithmic factor, which does not depend directly on $d$.
\end{remark}
\begin{remark}
      Lemma \ref{lemma0} illustrates that a sufficiently deep neural network can approximate the regression function very well. Based on Theorem 1 of \cite{Kohler2021}, a sufficiently wide neural network can also yield similar results. More specifically, replacing $L_n$ and $r_n$ with 
\begin{displaymath}
L_n=\left\lceil c_1 \cdot \log n\right\rceil, \  r_n=\left\lceil c_2 \cdot \max _{(p, K) \in \mathcal{P}} n^{\frac{K}{2(2 p+K)}}\right\rceil ,
\end{displaymath}
the Lemma \ref{lemma0} still holds.  
\end{remark}

We first consider a simple case of prediction of the fBm based on discrete observations, i.e., 
\begin{equation}
    \label{example1}
    m^1\left(X\right)=\mathbb{E}\left[B_{T}^H \mid \mathcal{F} ^{D,B^H,N}_s\right],
\end{equation}
where $\mathcal{F} ^{D,B^H,N}_s :=\sigma \overline{\left\{B_{t_1}^H, \dots, B_{t_{N}}^H\right\} }\ \text { for } 0 < t_1 < \dots < t_N= s < T$, $X=\left(B_{t_{1}}^H, \dots , B_{t_N}^H\right)$.

Obtaining the prediction formula for (\ref{example1}) is not difficult, because of the Gaussianity of the fBm, although the increments are not independent. The following result shows that there exists a deep neural network $m^1_n\left(\cdot\right) \in \mathcal{Y} \left(L_n, r_n\right)$ converging to $m^1\left(\cdot\right)$.

\begin{proposition}
\label{theorem1}
    Let $\left(B_t^H, \ t \in \left[0,T\right]\right)$ and $\left(B_{t,i}^H, \ t \in \left[0,T\right]\right)$, $i=1,\dots,n$, be independent fractional Brownian motions. Then $\left(X, Y\right), \left(X_1, Y_1\right), \ldots,\left(X_n, Y_n\right)$ are independent and identically distibuted random vectors, where
    \begin{displaymath}
    Y=B_T^H, \ Y_i=B_{T,i}^H, \ X=\left(B_{t_1}^H, \dots, B_{t_N}^H \right), \ X_i=\left(B_{t_1,i}^H, \dots, B_{t_N,i}^H \right).    
    \end{displaymath}
    Let $m^1_n=T_{c \log (n)} \tilde{m}^1_n$ for some $c>0$ sufficiently large, where
    \begin{displaymath}
    \tilde{m}^1_n(\cdot):=\arg \min _{y^{\theta} \in \mathcal{Y} \left(L_n, r_n\right)} \frac{1}{n} \sum_{i=1}^n\left|y^{\theta}\left(X_i\right)-Y_i\right|^2 .    
    \end{displaymath}
    Then 
    \begin{displaymath}
        \lim _{n \to \infty} \mathbb{E} \int\left|m^1_n(x)-m^1(x)\right|^2 \mathbf{P}_{X}(\mathrm{d} x)=0.
    \end{displaymath}    
\end{proposition}
\begin{proof}
$\left(X, Y\right)$ is multivariate normally distributed, hence we can obtain an explicit formulation of (\ref{example1}) by the property of the multivariate normal distribution: 
\begin{equation}
\label{exact1}
      m^1\left(X\right)=\Sigma_{12}^1 \left(\Sigma_{22}^1\right)^{-1}X^{\top},  
\end{equation}
where 
\begin{displaymath}
\left(\Sigma_{12}^1\right)^{\top}=\left(\begin{array}{c}
\vdots \\
\frac{1}{2}\left(T^{2 H}+t_{i}^{2 H}-|T-t_{i}|^{2 H}\right) \\
\vdots
\end{array}\right)\in \mathbb{R}^{N},
\end{displaymath}
\begin{displaymath}
\Sigma_{22}^1=\left(\frac{1}{2}\left(t_{i}^{2 H}+t_{j}^{2 H}-|t_{i}-t_{j}|^{2 H}\right)\right)_{i,j=1,\dots,N} \in \mathbb{S}^{N \times N}.   
\end{displaymath}
By Lemma \ref{lemma0}, it is clear that the regression function $m^1\left(\cdot\right) \in \mathcal{H}(l, \mathcal{P})$, and we complete the proof.
\end{proof}

The case of the integral of the fBm $Z_t=\int_0^t f(s) \mathrm{d} B^H_s$ is considered in the following, where $f$ is a bounded, measurable function. Although the stochastic process $Z_t$ is driven by the fBm, we actually cannot observe the path of $B^H_t$, so the conditional expectation that we are trying to work out is
\begin{equation}
    \label{example2}
    m^2\left(Z\right)=\mathbb{E}\left[Z_{T} \mid \mathcal{F} ^{D,Z,N}_s\right],
\end{equation}
where $\mathcal{F} ^{D,Z,N}_s :=\sigma \overline{\left\{Z_{t_1}, \dots, Z_{t_{N}}\right\} }$ for $0 < t_1 < \dots < t_N= s < T$, $Z=\left(Z_{t_{1}}, \dots , Z_{t_N}\right)$. Notice that $Z_t$ is Gaussian, so we can obtain the proposition for (\ref{example2}) in a similar way.
\begin{proposition}
    \label{theorem2}
    Let $\left(B_t^H, \ t \in \left[0,T\right]\right)$ and $\left(B_{t,i}^H, \ t \in \left[0,T\right]\right)$, $i=1,\dots,n$, be independent fractional Brownian motions. Then $\left(Z_t , \ t \in \left[0,T\right]\right)$ and $\left(Z_{t,i} , \ t \in \left[0,T\right]\right)$, $i=1,\dots,n$, are independent stochastic processes, and $\left(X, Y\right), \  \left(X_1, Y_1\right),  \ \ldots, \ \left(X_n, Y_n\right)$ are independent and identically distibuted random vectors, where 
    \begin{displaymath}
       Z_t=\int_0^t f(s) \mathrm{d} B^H_s, \ Z_{t,i} =\int_0^t f(s) \mathrm{d} B_{s,i}^H, 
    \end{displaymath}
    \begin{displaymath}
       Y=Z_T , \ Y_i=Z_{t,i} , \ X=\left(Z_{t_1} , \dots, Z_{t_N}  \right), \ X_i=\left(Z_{t_1,i} , \dots, Z_{t_N,i}  \right) .
    \end{displaymath}
   Let $m^2_n=T_{c \log (n)} \tilde{m}^2_n$ for some $c>0$ sufficiently large, where
    \begin{displaymath}
    \tilde{m}^2_n(\cdot):=\arg \min _{y^{\theta} \in \mathcal{Y} \left(L_n, r_n\right)} \frac{1}{n} \sum_{i=1}^n\left|y^{\theta}\left(X_i\right)-Y_i\right|^2 .
    \end{displaymath}
    Then 
    \begin{displaymath}
        \lim _{n \to \infty} \mathbb{E} \int\left|m^2_n(x)-m^2(x)\right|^2 \mathbf{P}_{X}(\mathrm{d} x)=0.
    \end{displaymath}  
\end{proposition}
\begin{proof}
$\left(X, Y\right)$ is multivariate normally distributed, hence we can obtain an explicit formulation of (\ref{example2}) by the property of the multivariate normal distribution: 
\begin{displaymath}
     m^2_n\left(Z\right)=\mathbb{E}\left[X_{T} \mid \mathcal{F} ^{D,Z,N}_s\right]=\Sigma_{12}^2 \left(\Sigma_{22}^2\right)^{-1}X^{\top},
\end{displaymath}
where 
\begin{displaymath}
\left(\Sigma_{12}^2\right)^{\top}=\left(\begin{array}{c}
\vdots \\
\operatorname{Cov}\left[\int_0^T f(s) \mathrm{d} B^H_s, \int_0^{t_i} f(s) \mathrm{d} B^H_s\right] \\
\vdots
\end{array}\right)\in \mathbb{R}^{N},
\end{displaymath}
\begin{displaymath}
\Sigma_{22}^2=\left(\operatorname{Cov}\left[\int_0^{t_i} f(s) \mathrm{d} B^H_s, \int_0^{t_j} f(s) \mathrm{d} B^H_s\right]\right)_{i,j=1,\dots,N} \in \mathbb{S}^{N \times N}.   
\end{displaymath}
By Lemma \ref{lemma0}, it is clear that the regression function $m^2\left(\cdot\right) \in \mathcal{H}(l, \mathcal{P})$, and we complete the proof.
\end{proof}

\begin{remark}
    Define $\mathcal{F} ^{C,Z}_s :=\sigma \overline{\left\{Z_{v}, \  v\in \left[0,s\right]\right\} }$. Notice that $\mathcal{F}_s^{C,Z} = \mathcal{F} ^{C,B^H}_s$ and therefore $\mathbb{E}\left[Z_{T} \mid \mathcal{F}_s^{C,Z}\right] = \mathbb{E}\left[Z_{T} \mid \mathcal{F} ^{C,B^H}_s\right]$. However, $\mathcal{F}_s^{D,Z,N} \neq \mathcal{F} ^{D,B^H,N}_s$ with the consequence that $\mathbb{E}\left[Z_{T} \mid \mathcal{F}_s^{D,Z,N}\right] \neq \mathbb{E}\left[Z_{T} \mid \mathcal{F} ^{D,B^H,N}_s\right]$. Beyond that, almost surely (a.s.) as $N \to \infty$ (see, e.g., \cite[Lemma 9.2.4.]{Dudley_2002}),
    \begin{displaymath}
        \mathbb{E}\left[Z_{T} \mid \mathcal{F}_s^{D,Z,N}\right] \to \mathbb{E}\left[Z_{T} \mid \mathcal{F}_s^{C,Z}\right].
    \end{displaymath}
    
    $\mathbb{E}\left[Z_{T} \mid \mathcal{F} ^{C,B^H}_s\right]$ plays a role in some research frameworks on the fBm, such as the martingale approach for the fBm and related path dependent partial differential equations proposed by \cite{Viens2017AMA} and the framework for analyzing stochastic volatility problems proposed by \cite{Garnier2017}. Calculating the element, however, presents two difficulties. Firstly, the theories use filtration generated from continuous observations, whereas real observations are always discrete. Secondly, the fBm driving $Z_t$ cannot be directly observed.
    
    The method proposed in the paper provides a way to address the problems. The deep neural network's input is observable $Z_t$ rather than $B_t^H$. In addition, as mentioned in the introduction, deep neural networks have advantages in processing high-dimensional data and can perform better in the case of large $N$ compared to traditional methods. These advantages can also be preserved as we further extend the applicability of the framework.
\end{remark}

\subsection{Prediction of discrete-time stochastic processes related to fractional Brownian motion}

The first consideration is to predict the solution of a linear stochastic differential equation with the fBm. Let $(A_t, \ t \geq 0)$ be the real-valued stochastic process that is the solution of the stochastic differential equation
\begin{equation}
\label{dfou}
\begin{aligned}
& \mathrm{d} A_t=\mu(t) A_t \mathrm{d} t+\mathrm{d} B_t^H, \\
& A_0=a_0,
\end{aligned}
\end{equation}
where $a_0 \in \mathbb{R}, \  \mu: \mathbb{R}_{+} \rightarrow \mathbb{R}$ is bounded and measurable. It can be verified that $A_t$ has an explicit solution 
\begin{equation}
\label{equ7}
    A_t=a_0e^{\int_0^t \mu(s)\mathrm{d} s} +\int_0^t \mathrm{e}^{\int_s^t \mu(s)\mathrm{d} s} \mathrm{d} B_s^H,
\end{equation}
and thus $\sigma \overline{\left\{A_{v}, \  v\in \left[0,s\right]\right\} }=\sigma \overline{\left\{B_{v}^H, \  v\in \left[0,s\right]\right\} } $. Notice that $A_t$ is Gaussian, so we can obtain the proposition for 
\begin{equation}
    \label{example3}
    m^3\left(A\right)=\mathbb{E}\left[A_{T} \mid \mathcal{F} ^{D,A,N}_s\right]
\end{equation}
in a similar way, where $\mathcal{F} ^{D,A,N}_s :=\sigma \overline{\left\{A_{t_1}, \dots, A_{t_{N}}\right\} }$ for $0   < t_1 < \dots < t_N= s < T$, $A=\left(A_{t_{1}}, \dots , A_{t_N}\right)$.
\begin{proposition}
    \label{theorem3}
    Let $\left(B_t^H, \ t \in \left[0,T\right]\right)$ and $\left(B_{t,i}^H, \ t \in \left[0,T\right]\right)$, $i=1,\dots,n$, be independent fractional Brownian motions. Then $\left(A_t , \ t \in \left[0,T\right]\right)$ and $\left(A_{t,i}^H, \ t \in \left[0,T\right]\right)$, $i=1,\dots,n$, are independent stochastic processes, and $\left(X, Y\right), \  \left(X_1, Y_1\right),  \ \ldots, \ \left(X_n, Y_n\right)$ are independent and identically distibuted random vectors, where
    \begin{displaymath}
    A_t=a_0e^{\int_0^t \mu(s)\mathrm{d} s} +\int_0^t \mathrm{e}^{\int_s^t \mu(s)\mathrm{d} s} \mathrm{d} B_s^H, \ A_{t,i}^H=a_0e^{\int_0^t \mu(s)\mathrm{d} s} +\int_0^t \mathrm{e}^{\int_s^t \mu(s)\mathrm{d} s} \mathrm{d} B_{s,i}^H, 
    \end{displaymath}
    \begin{displaymath}
        Y=A_T, \ Y_i=A_{T,i}, \ X=\left(A_{t_1}, \dots, A_{t_N} \right), \ X_i=\left(A_{t_1,i}, \dots, A_{t_N,i} \right).
    \end{displaymath}
     Let $m^3_n=T_{c \log (n)} \tilde{m}^3_n$ for some $c>0$ sufficiently large, where
    \begin{displaymath}
    \tilde{m}^3_n(\cdot):=\arg \min _{y^{\theta} \in \mathcal{Y} \left(L_n, r_n\right)} \frac{1}{n} \sum_{i=1}^n\left|y^{\theta}\left(X_i\right)-Y_i\right|^2 .    
    \end{displaymath}
    Then 
    \begin{displaymath}
        \lim _{n \to \infty} \mathbb{E} \int\left|m^3_n(x)-m^3(x)\right|^2 \mathbf{P}_{X}(\mathrm{d} x)=0.
    \end{displaymath}
\end{proposition}
\begin{proof}
Noting that $\left(X, Y\right)$ is also multivariate normally distributed, the proof of Proposition \ref{theorem3} is completed by Lemma \ref{lemma0}, since 
\begin{equation}
    \label{predictfou}
     m^3\left(A\right)=\mathbb{E}\left[A_{T} \mid \mathcal{F} ^{D,A,N}_s\right]=a_0e^{\int_0^T \mu(s)\mathrm{d} s} +\Sigma_{12}^3 \left(\Sigma_{22}^3\right)^{-1}X^{\top},    
\end{equation}
where 
\begin{displaymath}
\left(\Sigma_{12}^3\right)^{\top}=\left(\begin{array}{c}
\vdots \\
\operatorname{Cov}\left[\int_0^T \mathrm{e}^{\int_s^T \mu (s)\mathrm{d} s} \mathrm{d} B_s^H, \int_0^{t_i} \mathrm{e}^{\int_s^{t_i} \mu (s)\mathrm{d} s} \mathrm{d} B_s^H\right] \\
\vdots
\end{array}\right)\in \mathbb{R}^{N },
\end{displaymath}
\begin{displaymath}
\Sigma_{22}^3=\left(\operatorname{Cov}\left[\int_0^{t_i} \mathrm{e}^{\int_s^{t_i} \mu (s)\mathrm{d} s} \mathrm{d} B_s^H, \int_0^{t_j} \mathrm{e}^{\int_s^{t_j} \mu (s)\mathrm{d} s} \mathrm{d} B_s^H\right]\right)_{i,j=1,\dots,N} \in \mathbb{S}^{N \times N}.   
\end{displaymath}
and the regression function $m^3 \left(\cdot\right) \in \mathcal{H}(l, \mathcal{P})$
\end{proof}

\begin{remark}
    The Proposition \ref{theorem3} holds for all linear stochastic differential equations with the fBm, including fOU process, which is the solution of the stochastic differential equation
\begin{equation}
\label{FOU}
\begin{aligned}
& \mathrm{d} A_t=(k(t)-a(t) A_t) \mathrm{d} t+\sigma(t) \mathrm{d} B_t^H, \\
& A_0=a_0,
\end{aligned}    
\end{equation}
where $k$, $a$, $\sigma$ are bounded, measurable functions. 
\end{remark}

Consider a more general pathwise stochastic differential equation with the fBm, which can represent some more general non-Gaussian cases:
\begin{equation}
\label{equ10}
\begin{aligned}
& \mathrm{d} R_t=\mu(R_t) \mathrm{d} t+\sigma(R_t) \mathrm{d} B_t^H, \\
& R_0=r_0,
\end{aligned}  
\end{equation}
for suitable coefficient functions $\mu(\cdot)$, $\sigma(\cdot)$ and a constant $r_0$. For $H \in\left(\frac{1}{2}, 1\right)$, solutions to (\ref{equ10}) are given by
\begin{equation}
\label{equ11}
\begin{aligned}
& R_t=f(A_t), \\
& \mathrm{d} A_t=-a A_t \mathrm{d} t+\mathrm{d} B_t^H,\\
& A_0=f^{-1}(R_0),
\end{aligned}  
\end{equation}
for some monotone and differentiable $f: \mathbb{R} \rightarrow \mathbb{R}$ and $a>0$ (see, \cite{Boris2006}). For $H \in\left(1/2, 1\right)$, the fCIR process is the pathwise solution to the stochastic differential eqaution
\begin{equation}
\label{FCIRMODEL}
\begin{aligned}
& \mathrm{d} R_t=-\lambda R_t \mathrm{d} t+\sigma \sqrt{|R_t|} \mathrm{d} B_t^H,  \\
& R_0=r_0,
\end{aligned}    
\end{equation}
where $\lambda, \sigma,r_0>0$. A solution of (\ref{FCIRMODEL}) is given by
\begin{equation}
\label{fcir}
\begin{aligned}
& R_t=f(A_t),  \\
& \mathrm{d} A_t=-\frac{\lambda}{2} A_t \mathrm{d} t+\mathrm{d} B_t^H, \\
& A_0=f^{-1}(R_0),
\end{aligned}    
\end{equation}
where $f(x)=\operatorname{sgn}(x) \sigma^2 x^2 / 4$ (see, e.g., \cite[Proposition 5.7]{Boris2006}).

The fCIR process can be written as $f(X_t)$, where $f$ is a given function and $X_t$ is a Gaussian process, but $f$ is not a strictly monotone increasing function. It means that we cannot complete the proof in the same way as before. For a more specific reason, we assert that
\begin{equation}
    \label{example5}
    m^4\left(R\right)=\mathbb{E}\left[R_{T} \mid \mathcal{F} ^{D,R,N}_s\right]
\end{equation}
is not a continuous function anymore, where $\mathcal{F} ^{D,R,N}_s :=\sigma \overline{\left\{R_{t_1}, \dots, R_{t_{N}}\right\} }$ for $0  < t_1 < \dots < t_N=s<T$, $R=\left(R_{t_{1}}, \dots , R_{t_N}\right)$. To address this issue, we introduce the edge function based on the notion of piecewise smooth functions proposed by \cite{Imaizumi2019} and boundary fragment classes developed by \cite{Dudley1974}. In Appendix \ref{app:theorem}, we introduce a special class of functions that take 1 on some regions of the state space whose boundaries are formed by a series of smooth functions, and demonstrate the existence of neural networks that can approximate this class of functions.

\begin{proposition}
    \label{theorem4}
    Let $\left(B_t^H, \ t \in \left[0,T\right]\right)$ and $\left(B_{t,i}^H, \ t \in \left[0,T\right]\right)$, $i=1,\dots,n$, be independent fractional Brownian motions. Then $\left(R_t, \ t \in \left[0,T\right]\right)$ and $\left(R_{t,i}, \ t \in \left[0,T\right]\right)$, $i=1,\dots,n$, are independent stochastic processes, and $\left(X, Y\right), \  \left(X_1, Y_1\right),  \ \ldots, \ \left(X_n, Y_n\right)$ are independent and identically distibuted random vectors, where $R_t$ and $R_{t,i}$, $i=1,\dots,n$, are respectively solutions of stochastic differential eqautions (\ref{fcir}) with $B_t^H$ and $B_{t,i}^H$ for $i=1,\dots,n$,
    \begin{displaymath}
        Y=R_T, \ Y_i=R_{T,i}, \ X=\left(R_{t_1}, \dots, R_{t_N} \right), \ X_i=\left(R_{t_1,i}, \dots, R_{t_N,i} \right).
    \end{displaymath}
     Let $m^4_n=T_{c \log (n)} \tilde{m}^4_n$ for some $c>0$ sufficiently large, where
    \begin{displaymath}
    \tilde{m}^4_n(\cdot):=\arg \min _{y^{\theta} \in \mathcal{Y} \left(L_n+L, r_n+r\right)} \frac{1}{n} \sum_{i=1}^n\left|y^{\theta}\left(X_i\right)-Y_i\right|^2     
    \end{displaymath}
    for some $L, \ r$ sufficiently large. Then 
    \begin{displaymath}
        \lim _{n \to \infty} \mathbb{E} \int\left|m^4_n(x)-m^4(x)\right|^2 \mathbf{P}_{X}(\mathrm{d} x)=0.
    \end{displaymath}
\end{proposition}
\begin{proof}
For $r=\left(r_1,\ldots,r_N\right) \in \left[0,+\infty\right)^{N}$, $m^4\left(r\right)=\mathbb{E}\left[R_{T} \mid R_{t_1}=r_1, \ldots, R_{t_N}=r_N\right]$. When $r_1,\ldots,r_N > 0$,
\begin{displaymath}
\begin{aligned}
    m^4\left(r\right)&=\mathbb{E}\left[R_{T} \mid R_{t_1}=r_1, \ldots, R_{t_N}=r_N\right]  \\
    &=\mathbb{E}\left[\frac{\operatorname{sgn}(A_T)\sigma^2 A_T^2}{4} \mid A_{t_1}=\sqrt{\frac{4r_1}{\sigma^2}} , \ldots, A_{t_N}=\sqrt{\frac{4r_N}{\sigma^2}}\right] \\
    &=\int_{0}^{+\infty } \frac{\sigma ^2 u^2}{4\sqrt{2\pi \widehat{\sigma }^2 }   } e^{-\frac{\left (u-\widehat{\mu } \right )^2}{2\widehat{\sigma }^2 } }\mathrm{d}u,
\end{aligned}
\end{displaymath}
where 
\begin{displaymath}
\begin{aligned}
    \widehat{\mu }&=\mathbb{E}\left[A_{T} \mid A_{t_1}=\sqrt{\frac{4r_1}{\sigma^2}} , \ldots, A_{t_N}=\sqrt{\frac{4r_N}{\sigma^2}}\right]\\
    &=\sqrt{\frac{4r_0}{\sigma^2}}e^{ -\frac{\lambda T}{2}} +\Sigma_{12}^3 \left(\Sigma_{22}^3\right)^{-1}\left(a-\mu_2\right)^{\top},
\end{aligned}     
\end{displaymath}
\begin{displaymath}
\begin{aligned}
    \widehat{\sigma }^2&=\operatorname{Var}\left[A_{T} \mid A_{t_1}=\sqrt{\frac{4r_1}{\sigma^2}} , \ldots, A_{t_N}=\sqrt{\frac{4r_N}{\sigma^2}}\right]\\
    &=\Sigma_{11}^3-\Sigma_{12}^3 \left(\Sigma_{22}^3\right)^{-1} \Sigma_{21},
\end{aligned}     
\end{displaymath}
\begin{displaymath}
a=\left(\begin{array}{c}
\vdots \\
\sqrt{\frac{4r_i}{\sigma^2}} \\
\vdots
\end{array}\right),\
\mu_2=\left(\begin{array}{c}
\vdots \\
\sqrt{\frac{4r_0}{\sigma^2}}e^{ -\frac{\lambda t_i}{2}} \\
\vdots
\end{array}\right)\in \mathbb{R}^{N}, \ \Sigma_{11}^3=\operatorname{Var}\left[\int_0^T \mathrm{e}^{ -\frac{\lambda \left(T-s\right)}{2}} \mathrm{d} B_s^H\right],
\end{displaymath}
\begin{displaymath}
\left(\Sigma_{12}^3\right)^{\top}=\left(\Sigma_{21}^3\right)=\left(\begin{array}{c}
\vdots \\
\operatorname{Cov}\left[\int_0^T e^{ -\frac{\lambda \left(T-s\right)}{2}} \mathrm{d} B_s^H, \int_0^{t_i} e^{ -\frac{\lambda \left(t_i-s\right)}{2}} \mathrm{d} B_s^H\right] \\
\vdots
\end{array}\right)\in \mathbb{R}^{N},
\end{displaymath}
\begin{displaymath}
\Sigma_{22}^3=\left(\operatorname{Cov}\left[\int_0^{t_i} e^{ -\frac{\lambda \left(t_i-s\right)}{2}} \mathrm{d} B_s^H, \int_0^{t_j} e^{ -\frac{\lambda \left(t_j-s\right)}{2}} \mathrm{d} B_s^H\right]\right)_{i,j=1,\dots,N} \in \mathbb{S}^{N \times N}.   
\end{displaymath}

Here we consider the case where not all $r_j$, $j=1,\ldots,N$are greater than zero. Without loss of generality, we set $r_1,\ldots,r_M>0$ and $r_{M+1},\ldots,r_N=0$.
{\small
\begin{displaymath}
\begin{aligned}
    m^4\left(r\right)&=\mathbb{E}\left[R_{T} \mid R_{t_1}=r_1, \ldots, R_{t_N}=r_N\right]  \\
        &=\mathbb{E}\left[\frac{\operatorname{sgn}(A_T)\sigma^2 A_T^2}{4} \mid A_{t_1}=\sqrt{\frac{4r_1}{\sigma^2}} , \ldots, A_{t_M}=\sqrt{\frac{4r_M}{\sigma^2}},A_{t_{M+1}}\le 0, \ldots, A_{t_N}\le 0\right] \\
    &=\int_{0}^{+\infty } \frac{\sigma ^2 u^2}{4}\frac{\int_{-\infty}^{0}\ldots\int_{-\infty}^{0} f_x\left(u,a_1, \ldots,a_M,x_{M+1},\ldots, x_N\right) \mathrm{d}x_{M+1} \ldots \mathrm{d}x_N}{\int_{-\infty}^{+\infty}\int_{-\infty}^{0}\ldots\int_{-\infty}^{0} f_x \left(x_T,a_1, \ldots,a_M,x_{M+1},\ldots, x_N\right) \mathrm{d}x_{M+1} \ldots \mathrm{d}x_N\mathrm{d}x_T}\mathrm{d}u,\\
    &=C\int_{0}^{+\infty }\int_{-\infty}^{0}\ldots\int_{-\infty}^{0}u^2f_x\left(u,a_1, \ldots,a_M,x_{M+1},\ldots, x_N\right)\mathrm{d}x_{M+1} \ldots \mathrm{d}x_N\mathrm{d}u,
\end{aligned}
\end{displaymath}
}
where $C$ is some constant, $a_i=\sqrt{\frac{4r_i}{\sigma^2}}$ for $i=1,\ldots,M$,
\begin{displaymath}
    f_x\left(x_T,x_1, \ldots x_N\right)=\frac{1}{\sqrt{(2 \pi)^{N+1}|\Sigma|}} e^{-\frac{\left(x-\mu\right)^{\top} \Sigma^{-1}\left(x-\mu\right)}{2}},
\end{displaymath}
\begin{displaymath}
x=\left(\begin{array}{c}
x_T\\
x_1\\
\vdots \\
x_i \\
\vdots
\end{array}\right),\
\mu=\left(\begin{array}{l}
\mu_1 \\
\mu_2
\end{array}\right) ,\
\Sigma=\left(\begin{array}{ll}
\Sigma^3_{11} & \Sigma^3_{12} \\
\Sigma^3_{21} & \Sigma^3_{22}
\end{array}\right), \ 
\mu_1=\sqrt{\frac{4r_0}{\sigma^2}}e^{ -\frac{\lambda T}{2}}.
\end{displaymath}

Although $m^4\left(\cdot\right)$ is not a $\left(p, C\right)$-smooth function or even continuous function in $ \left[0,+\infty\right)^{N}$. But it is able to divide $\left[0,+\infty\right)^{N}$ into regions such that in each region $m^4\left(\cdot\right)$ is a $\left(p, C\right)$-smooth function. Further, in each region, $m^4\left(\cdot\right) \in \mathcal{H}(l, \mathcal{P})$. By Lemma \ref{lemma4} and Theorem 1 of \cite{Kohler2021}, there exist a neural network $ \left(L, r\right)$ that makes it possible to divide $\left[0,+\infty\right)^{N}$ into regions, and a neural network ${y^{\theta} \in \mathcal{Y} \left(L_n, r_n\right)}$ that approximates the regression function in each region. Concatenate two neural networks in series and the proof is complete.
\end{proof}

\begin{remark}
    From the above discussion, we can conclude that our proposed framework for predicting path-dependent processes $W_t$ is broadly applicable to situations where the relationship between discrete observations of $V_t$ and $W_t$ can be described by a series of $\left(p, C\right)$-smooth functions. Therefore, we can also extend the framework to some cases where $V_t\ne W_t$. A simple example is that observed historical information is attached with independent noise, similar to Kalman filtering, e.g., $W_t=B_t^H, \ V_t=B_t^H+B_t$, where $B_t^H$ and $B_t$ are independent. Moreover, Lemma \ref{lemma4} allows us to divide several input regions and consider the input-output relationship mentioned in Remark \ref{remark4} case by case.
\end{remark}

\section{Numerical simulation}
\label{sec:4}
In this section, we aim to demonstrate the effectiveness of the framework proposed in the paper through numerical simulation. We apply the method to two examples: the fBm and the fOU process. We want to show that the method is accurate under discrete observations. A method is to generate independent sample paths $\left(S_{t,i}, \ t\in\left[0,T\right]\right)$ and use a trained neural network $\hat{y}^{\theta}$ to predict each path separately. We can get the error between the predicted value $\hat{y}^{\theta}\left(S_{t_1,i},\ldots,S_{t_N,i}\right)$ and the true value of the sample paths $S_{T,i}$, and calculate the mean error (ME)
\begin{displaymath}
    \frac{1}{N'}\sum_{i=1}^{N'}\left(S_{T,i}-\hat{y}^{\theta}\left(S_{t_1,i},\ldots,S_{t_N,i}\right)\right)
\end{displaymath}
and mean square error (MSE) 
\begin{displaymath}
    \frac{1}{N'}\sum_{i=1}^{N'}\left(S_{T,i}-\hat{y}^{\theta}\left(S_{t_1,i},\ldots,S_{t_N,i}\right)\right)^2.
\end{displaymath}
Using MSE as an evaluation metric, we compare the theoretically optimal predictor with the obtained deep neural network and also analyze the impact of factors such as prediction period, Hurst index, etc. on the accuracy.

Moreover, we intend to demonstrate that as $N \to \infty$, we can approximate the prediction under continuous observation $\mathbb{E}\left[S_{T} \mid \mathcal{F} ^{C,S}_s\right]$ with the prediction under discrete observation $\hat{y}^{\theta}\left(S_{t_1},\ldots,S_{t_N}\right)$. We therefore present some prediction theories with continuous filtrations and use the theoretical output as benchmarks.

In the deep neural network training process, for each case, we trained 3000 batches, generating $2^{12}$ sample paths per batch for training. We adopt a decreasing learning rate strategy, with an initial learning rate of 0.01, and the learning rate decreases to 0.95 of the original learning rate for every 10 training batches. We generate 10000 independent samples of each to test the method, i.e., $N'=10000$.

\subsection{Fractional Brownian motion}
In Table \ref{table1}, we demonstrate the results of predicting the fBms with different numbers of equally spaced grid points $N$, Hurst index $H$ and $s$. The MEs fluctuate within a small range around zero. This suggests that, as a whole, our predictions are not yielding large shifts. The MSEs show a decreasing trend as $s$ increases. This is consistent with the fact that the farther into the future it is, the more difficult it is to predict accurately. It is worth noting that the MSEs of the predictions are not significantly different as the number of equally spaced grid points changes. This suggests that not the more historical information is available from discrete observations, the smaller the prediction bias. Discrete observations here refer to the entirety of the historical information available to us, to be distinguished from the case where a portion of our known information is taken for prediction purposes.

Taking the $H=0.5$ case as a comparison, i.e., the sBm case, firstly the MSE is larger in $H=0.1,0.3,0.9$ cases but smaller in $H=0.7$ case, which implies that the path dependence may lead to larger or smaller prediction deviations. When $s = 2$, the MSE in the $H = 0.1$ case is smaller than the MSE in the $H = 0.9$ case. When $s=8$, the situation is reversed. The forecast time horizon affects the role of the Hurst index on forecast accuracy. In Table \ref{table2}, let $s=5$, $N=2^{12}$ and adjust $H, \ T$. As $T$ increases from 5.5 to 10, the MSE minimum case changes from $H = 0.8$ to $H = 0.7$, and the MSE difference multiples for the $H=0.1,0.9$ cases become smaller. This is consistent with the analysis results derived from Table \ref{table1}. 

\begin{table}
\centering
\resizebox{1.0\linewidth}{!}{
\begin{tabular}{crrrrrrrrrr}
\hline
$N$        & \multicolumn{1}{c}{ME} & \multicolumn{1}{c}{MSE} & \multicolumn{1}{c}{ME} & \multicolumn{1}{c}{MSE} & \multicolumn{1}{c}{ME} & \multicolumn{1}{c}{MSE} & \multicolumn{1}{c}{ME} & \multicolumn{1}{c}{MSE} & \multicolumn{1}{c}{ME} & \multicolumn{1}{c}{MSE} \\ \hline
         & \multicolumn{2}{c}{$H=0.1$}                        & \multicolumn{2}{c}{$H=0.3$}                        & \multicolumn{2}{c}{$H=0.5$}                        & \multicolumn{2}{c}{$H=0.7$}                        & \multicolumn{2}{c}{$H=0.9$}                        \\ \hline
         & \multicolumn{10}{c}{$s=2$}                                                                                                                                                                                                                                     \\ \hline
$2^{9}$  & -0.042                 & 20.978                  & 0.020                  & 19.517                  & 0.030                  & 18.370                  & -0.010                 & 17.141                  & 0.035                  & 27.022                  \\
$2^{10}$ & 0.043                  & 20.993                  & 0.004                  & 20.456                  & 0.035                  & 18.519                  & -0.054                 & 17.136                  & 0.037                  & 26.476                  \\
$2^{11}$ & -0.026                 & 21.382                  & 0.007                  & 20.287                  & 0.052                  & 18.504                  & 0.008                  & 16.835                  & 0.024                  & 26.511                  \\
$2^{12}$ & 0.021                  & 20.871                  & -0.007                 & 20.240                  & 0.026                  & 18.484                  & 0.007                  & 16.730                  & -0.105                 & 26.968                  \\
$2^{13}$ & -0.079                 & 21.081                  & 0.043                  & 20.022                  & -0.025                 & 18.056                  & 0.071                  & 17.121                  & 0.030                  & 26.553                  \\
$2^{14}$ & 0.084                  & 20.726                  & 0.031                  & 19.766                  & -0.005                 & 18.648                  & -0.034                 & 16.531                  & -0.012                 & 26.655                  \\
$2^{15}$ & 0.065                  & 20.463                  & -0.050                 & 19.930                  & 0.007                  & 18.516                  & 0.041                  & 16.719                  & -0.023                 & 26.131                  \\
$2^{16}$ & -0.129                 & 21.074                  & 0.021                  & 20.306                  & 0.011                  & 18.071                  & -0.012                 & 16.902                  & 0.056                  & 26.460                  \\ \hline
         & \multicolumn{10}{c}{$s=5$}                                                                                                                                                                                                                                     \\ \hline
$2^{9}$  & 0.094                  & 13.660                  & 0.079                  & 11.591                  & -0.017                 & 9.550                   & 0.061                  & 8.349                   & 0.012                  & 11.338                  \\
$2^{10}$ & -0.062                 & 13.749                  & 0.051                  & 11.485                  & -0.007                 & 9.414                   & -0.011                 & 8.350                   & -0.022                 & 10.843                  \\
$2^{11}$ & -0.027                 & 13.488                  & -0.026                 & 11.609                  & -0.004                 & 9.466                   & -0.028                 & 8.399                   & -0.006                 & 11.045                  \\
$2^{12}$ & -0.027                 & 13.578                  & 0.012                  & 11.734                  & 0.030                  & 9.435                   & -0.022                 & 8.267                   & 0.058                  & 10.828                  \\
$2^{13}$ & 0.026                  & 13.748                  & -0.027                 & 11.584                  & 0.035                  & 9.544                   & 0.012                  & 8.273                   & -0.005                 & 11.026                  \\
$2^{14}$ & -0.011                 & 13.775                  & -0.025                 & 11.511                  & -0.003                 & 9.418                   & 0.015                  & 8.309                   & 0.025                  & 11.015                  \\
$2^{15}$ & 0.059                  & 13.603                  & 0.002                  & 11.704                  & -0.004                 & 9.582                   & 0.029                  & 8.194                   & 0.012                  & 11.009                  \\
$2^{16}$ & 0.013                  & 13.738                  & -0.016                 & 12.009                  & -0.034                 & 9.506                   & -0.037                 & 8.268                   & 0.014                  & 11.115                  \\ \hline
         & \multicolumn{10}{c}{$s=8$}                                                                                                                                                                                                                                     \\ \hline
$2^{9}$  & -0.042                 & 5.266                   & 0.017                  & 3.689                   & -0.025                 & 2.585                   & 0.020                  & 2.247                   & 0.006                  & 2.635                   \\
$2^{10}$ & 0.010                  & 5.407                   & 0.013                  & 3.655                   & -0.023                 & 2.617                   & -0.016                 & 2.267                   & 0.034                  & 2.656                   \\
$2^{11}$ & 0.000                  & 5.424                   & -0.015                 & 3.684                   & 0.010                  & 2.699                   & -0.039                 & 2.277                   & 0.022                  & 2.816                   \\
$2^{12}$ & 0.029                  & 5.424                   & 0.023                  & 3.632                   & 0.011                  & 2.627                   & -0.013                 & 2.300                   & -0.004                 & 2.693                   \\
$2^{13}$ & 0.028                  & 5.294                   & 0.001                  & 3.745                   & 0.019                  & 2.643                   & -0.010                 & 2.327                   & 0.004                  & 2.765                   \\
$2^{14}$ & 0.024                  & 5.356                   & 0.010                  & 3.654                   & -0.001                 & 2.625                   & -0.028                 & 2.237                   & -0.010                 & 2.683                   \\
$2^{15}$ & 0.001                  & 5.140                   & -0.039                 & 3.692                   & -0.018                 & 2.622                   & 0.002                  & 2.254                   & -0.030                 & 2.696                   \\
$2^{16}$ & 0.019                  & 5.451                   & 0.014                  & 3.687                   & 0.020                  & 2.632                   & -0.005                 & 2.296                   & -0.002                 & 2.705                   \\ \hline
\end{tabular}
}
\caption{MEs and MSEs in the fBm cases with different $N, \ H, \ s$ for $T=10$.}
\label{table1}
\end{table}

\begin{table}
\centering
\begin{tabular}{lrrrrrrrrr}
\hline
\multicolumn{1}{c}{$H$} & \multicolumn{1}{c}{0.1} & \multicolumn{1}{c}{0.2} & \multicolumn{1}{c}{0.3} & \multicolumn{1}{c}{0.4} & \multicolumn{1}{c}{0.5} & \multicolumn{1}{c}{0.6} & \multicolumn{1}{c}{0.7} & \multicolumn{1}{c}{0.8} & \multicolumn{1}{c}{0.9} \\ \hline
$T=5.5$                 & 1.046                   & 0.764                   & 0.582                   & 0.471                   & 0.383                   & 0.338                   & 0.332                   & 0.331                   & 0.379                   \\
$T=6$                   & 2.226                   & 1.721                   & 1.412                   & 1.179                   & 1.004                   & 0.892                   & 0.860                   & 0.887                   & 1.036                   \\
$T=6.5$                 & 3.417                   & 2.847                   & 2.402                   & 2.053                   & 1.775                   & 1.612                   & 1.527                   & 1.561                   & 1.816                   \\
$T=7$                   & 4.773                   & 4.040                   & 3.441                   & 3.020                   & 2.646                   & 2.394                   & 2.270                   & 2.308                   & 2.768                   \\
$T=7.5$                 & 6.127                   & 5.356                   & 4.817                   & 4.170                   & 3.554                   & 3.273                   & 3.169                   & 3.273                   & 3.926                   \\
$T=8$                   & 7.651                   & 6.617                   & 5.974                   & 5.235                   & 4.578                   & 4.254                   & 4.033                   & 4.170                   & 5.127                   \\
$T=8.5$                 & 8.929                   & 8.204                   & 7.418                   & 6.385                   & 5.827                   & 5.260                   & 5.215                   & 5.459                   & 6.604                   \\
$T=9$                   & 10.638                  & 9.574                   & 8.744                   & 7.916                   & 6.921                   & 6.361                   & 5.992                   & 6.515                   & 7.886                   \\
$T=9.5$                 & 12.012                  & 11.419                  & 10.224                  & 9.246                   & 8.344                   & 7.651                   & 7.136                   & 7.662                   & 9.594                   \\
$T=10$                  & 13.487                  & 12.350                  & 11.702                  & 10.616                  & 9.625                   & 8.658                   & 8.557                   & 9.006                   & 10.738                  \\ \hline
\end{tabular}
\caption{MSEs in the fBm cases with different $H, \ T$ for $s=5$, $N=2^{12}$.}
\label{table2}
\end{table}

In Figure \ref{figure1}, we compare the MSEs of the theoretically optimal predictor (\ref{exact1}) with the MSEs of the trained deep neural network. The MSE gap, although slightly larger with increasing $T$, is always confined within a small range near 0. It demonstrates that the predictions obtained through our proposed framework are not significantly different from the theoretical optimal predictions, hence demonstrating the reliability of the method.

\begin{figure}[htbp]
\centering
\includegraphics[width=0.95\textwidth,trim=40 0 40 0,clip]{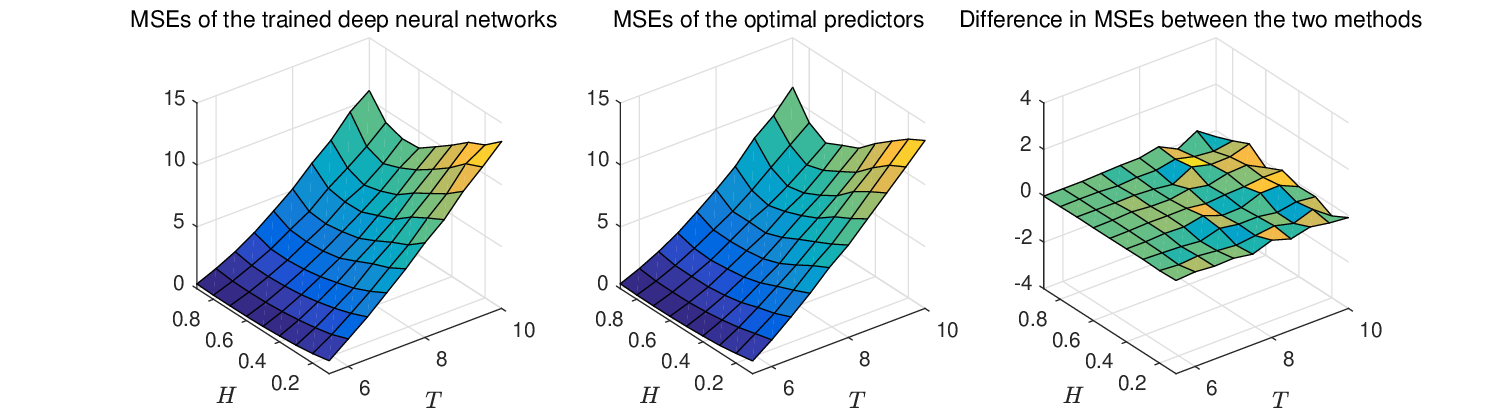}
\caption{MSEs of the two methods in the fBm cases for $s=5$, $N=2^{12}$ and difference.}
\label{figure1}
\end{figure}

Finally, given sample paths, predictions are made based on different numbers of equally spaced grid points in Figure \ref{figure2}. Consider the prediction of the fBm with continuous observations, i.e., 
\begin{displaymath}
    \mathbb{E}\left[B_{T}^H \mid \mathcal{F} ^{C,B^H}_s\right].
\end{displaymath}
This prediction problem is well solved by \cite{gripenberg_norros_1996} and \cite{Pipiras2001}:
\begin{displaymath}
    \mathbb{E}\left[B_{T}^H \mid \mathcal{F} ^{C,B^H}_s\right]=B^H_s+\int_0^s \Psi^H(s, T, v) \mathrm{d} B^H_v,
\end{displaymath}
where, for $v \in(0, s)$,
\begin{displaymath}
    \Psi^H(s, T, v)=\frac{\sin\left(\left(H- \frac{1}{2}\right)\pi  \right)}{\pi} v^{-H+ \frac{1}{2}}(s-v)^{-H+ \frac{1}{2}} \int_s^T \frac{z^{H- \frac{1}{2}}(z-s)^{H- \frac{1}{2}}}{z-v} \mathrm{d} z,
\end{displaymath}
and, for $v \in\{0, s\}$, we have $\Psi^H(s, T, v)=0$. 

Let $s=5$ and $T=10$. Given sample paths, Figure \ref{figure2} shows that the predictions based on discrete observations obtained through the deep neural network method converge to predictions based on continuous observations. When $H=0.1,0.3$, convergence is faster.

\begin{figure}[htbp]
\centering
\includegraphics[width=0.90\textwidth,trim=20 0 0 0,clip]{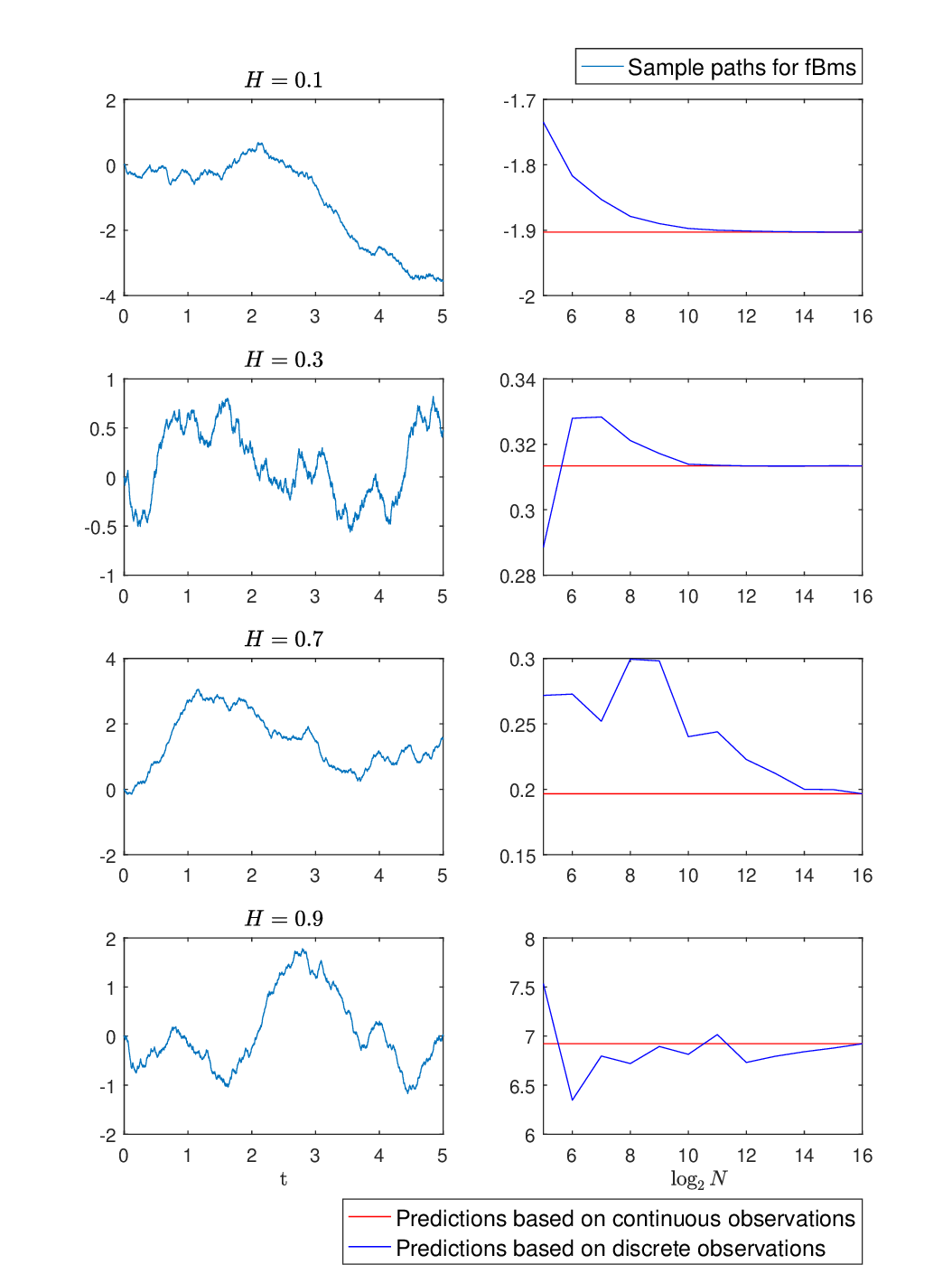}
\caption{Sample paths, predictions based on discrete observations and continuous observations in the fBm cases for $s=5$, $T=10$.}
\label{figure2}
\end{figure}

\subsection{Fractional Ornstein-Uhlenbeck process}
 Let $k(\cdot)=a_0=0,\ a(\cdot)=\frac{1}{2}$, and $\sigma(\cdot)=1$. In Table \ref{table3}, we demonstrate the results of predicting the fOU processes with different numbers of equally spaced grid points $N$, Hurst index $H$ and $s$. The MEs also fluctuate within a small range around zero. This suggests that, as a whole, our predictions are not yielding large shifts. The MSEs show a decreasing trend as $s$ increases. It is consistent with the fact that the farther into the future it is, the more difficult it is to predict accurately. The MSEs of the predictions are also not significantly different as the number of equally spaced grid points changes. 

It is not the case that the larger or smaller the $H$, the smaller the MSE is. When $s=2$, the MSEs are minimum for the $H=0.7$ case. When $s=5,8$, the MSEs decrease with increasing $H$. It implies that the path dependence may lead to larger or smaller prediction deviations and the forecast time horizon jointly affects them. In Table \ref{table4}, let $s=5$, $N=2^{12}$ and adjust $H, \ T$. The MSEs decrease with increasing $H$ and the MSE difference multiples for the $H=0.1,0.9$ cases become smaller. This is consistent with the analysis results derived from Table \ref{table3}.

\begin{table}
\centering
\resizebox{1.0\linewidth}{!}{
\begin{tabular}{crrrrrrrrrr}
\hline
$N$        & \multicolumn{1}{c}{ME} & \multicolumn{1}{c}{MSE} & \multicolumn{1}{c}{ME} & \multicolumn{1}{c}{MSE} & \multicolumn{1}{c}{ME} & \multicolumn{1}{c}{MSE} & \multicolumn{1}{c}{ME} & \multicolumn{1}{c}{MSE} & \multicolumn{1}{c}{ME} & \multicolumn{1}{c}{MSE} \\ \hline
         & \multicolumn{2}{c}{$H=0.1$}                      & \multicolumn{2}{c}{$H=0.3$}                      & \multicolumn{2}{c}{$H=0.5$}                      & \multicolumn{2}{c}{$H=0.7$}                      & \multicolumn{2}{c}{$H=0.9$}                      \\ \hline
         & \multicolumn{10}{c}{$s=2$}                                                                                                                                                                                                                                   \\ \hline
$2^{9}$  & 0.012                  & 16.003                  & 0.014                  & 15.655                  & 0.038                  & 14.610                  & 0.004                  & 14.145                  & -0.009                 & 16.472                  \\
$2^{10}$ & 0.054                  & 15.824                  & 0.015                  & 15.660                  & -0.009                 & 14.486                  & 0.054                  & 13.759                  & -0.057                 & 16.629                  \\
$2^{11}$ & 0.013                  & 15.665                  & -0.014                 & 16.038                  & -0.002                 & 14.697                  & -0.020                 & 14.175                  & 0.032                  & 16.414                  \\
$2^{12}$ & -0.031                 & 15.446                  & 0.020                  & 15.281                  & 0.036                  & 14.617                  & 0.059                  & 14.111                  & -0.063                 & 16.872                  \\
$2^{13}$ & 0.074                  & 15.559                  & 0.026                  & 15.730                  & 0.027                  & 14.163                  & -0.024                 & 13.932                  & -0.105                 & 16.244                  \\
$2^{14}$ & 0.015                  & 15.528                  & -0.058                 & 15.359                  & -0.020                 & 14.545                  & -0.007                 & 13.921                  & 0.053                  & 16.476                  \\
$2^{15}$ & 0.014                  & 15.960                  & 0.071                  & 15.770                  & 0.050                  & 14.657                  & 0.055                  & 14.232                  & 0.002                  & 16.523                  \\
$2^{16}$ & 0.077                  & 15.591                  & 0.019                  & 15.295                  & 0.001                  & 14.681                  & 0.040                  & 13.827                  & -0.020                 & 16.627                  \\ \hline
         & \multicolumn{10}{c}{$s=5$}                                                                                                                                                                                                                                   \\ \hline
$2^{9}$  & -0.024                 & 15.329                  & -0.031                 & 14.383                  & -0.034                 & 13.195                  & 0.025                  & 10.013                  & 0.022                  & 7.659                   \\
$2^{10}$ & -0.038                 & 15.397                  & 0.042                  & 14.268                  & -0.031                 & 13.351                  & -0.043                 & 10.139                  & 0.025                  & 7.670                   \\
$2^{11}$ & 0.040                  & 15.004                  & 0.044                  & 14.226                  & -0.048                 & 13.290                  & 0.011                  & 10.193                  & 0.003                  & 7.604                   \\
$2^{12}$ & 0.026                  & 14.708                  & -0.004                 & 13.706                  & 0.046                  & 13.160                  & 0.004                  & 10.123                  & 0.033                  & 7.521                   \\
$2^{13}$ & -0.018                 & 14.973                  & 0.028                  & 14.225                  & -0.036                 & 13.328                  & -0.025                 & 10.059                  & -0.013                 & 7.355                   \\
$2^{14}$ & 0.016                  & 15.019                  & 0.012                  & 14.242                  & -0.036                 & 13.348                  & 0.011                  & 10.320                  & 0.025                  & 7.644                   \\
$2^{15}$ & -0.019                 & 14.934                  & 0.046                  & 14.364                  & -0.012                 & 13.273                  & -0.040                 & 10.194                  & -0.028                 & 7.680                   \\
$2^{16}$ & 0.035                  & 14.929                  & 0.044                  & 14.359                  & 0.023                  & 13.246                  & -0.068                 & 10.294                  & 0.003                  & 7.702                   \\ \hline
         & \multicolumn{10}{c}{$s=8$}                                                                                                                                                                                                                                   \\ \hline
$2^{9}$  & 0.006                  & 10.017                  & -0.055                 & 8.193                   & -0.015                 & 5.597                   & -0.012                 & 3.591                   & -0.023                 & 2.402                   \\
$2^{10}$ & -0.040                 & 10.182                  & 0.022                  & 8.188                   & -0.003                 & 5.646                   & -0.012                 & 3.742                   & 0.008                  & 2.347                   \\
$2^{11}$ & -0.007                 & 10.423                  & -0.005                 & 8.254                   & 0.011                  & 5.608                   & -0.004                 & 3.660                   & -0.015                 & 2.352                   \\
$2^{12}$ & -0.047                 & 10.278                  & -0.004                 & 8.130                   & -0.004                 & 5.518                   & 0.038                  & 3.588                   & -0.009                 & 2.358                   \\
$2^{13}$ & -0.032                 & 10.299                  & 0.001                  & 8.116                   & 0.009                  & 5.592                   & 0.012                  & 3.671                   & -0.020                 & 2.349                   \\
$2^{14}$ & -0.005                 & 10.126                  & 0.000                  & 7.986                   & -0.020                 & 5.594                   & 0.008                  & 3.573                   & 0.001                  & 2.440                   \\
$2^{15}$ & 0.036                  & 10.177                  & 0.011                  & 8.190                   & 0.006                  & 5.604                   & 0.013                  & 3.635                   & 0.010                  & 2.404                   \\
$2^{16}$ & -0.002                 & 10.400                  & 0.021                  & 8.230                   & -0.021                 & 5.576                   & 0.013                  & 3.692                   & 0.004                  & 2.361                   \\ \hline
\end{tabular}
}
\caption{MEs and MSEs in the fOU process cases with different $N, \ H, \ s$ for $T=10$.}
\label{table3}
\end{table}

\begin{table}
\centering
\begin{tabular}{lrrrrrrrrr}
\hline
\multicolumn{1}{c}{$H$} & \multicolumn{1}{c}{0.1} & \multicolumn{1}{c}{0.2} & \multicolumn{1}{c}{0.3} & \multicolumn{1}{c}{0.4} & \multicolumn{1}{c}{0.5} & \multicolumn{1}{c}{0.6} & \multicolumn{1}{c}{0.7} & \multicolumn{1}{c}{0.8} & \multicolumn{1}{c}{0.9} \\ \hline
$T=5.5$                 & 1.561                   & 1.268                   & 0.957                   & 0.755                   & 0.453                   & 0.436                   & 0.376                   & 0.335                   & 0.367                   \\
$T=6$                   & 3.320                   & 2.945                   & 2.359                   & 2.004                   & 1.524                   & 1.262                   & 1.034                   & 0.850                   & 0.849                   \\
$T=6.5$                 & 5.106                   & 4.504                   & 3.963                   & 3.345                   & 2.773                   & 2.341                   & 1.848                   & 1.576                   & 1.443                   \\
$T=7$                   & 6.859                   & 6.273                   & 5.627                   & 4.937                   & 4.321                   & 3.551                   & 2.866                   & 2.370                   & 2.142                   \\
$T=7.5$                 & 8.428                   & 7.829                   & 7.247                   & 6.488                   & 5.897                   & 4.753                   & 3.788                   & 3.249                   & 2.874                   \\
$T=8$                   & 9.727                   & 9.270                   & 8.605                   & 7.894                   & 7.333                   & 6.087                   & 5.189                   & 4.307                   & 3.802                   \\
$T=8.5$                 & 11.121                  & 10.630                  & 10.094                  & 9.333                   & 8.820                   & 7.566                   & 6.281                   & 5.379                   & 4.534                   \\
$T=9$                   & 12.522                  & 12.161                  & 11.575                  & 10.900                  & 10.426                  & 8.919                   & 7.605                   & 6.400                   & 5.603                   \\
$T=9.5$                 & 13.747                  & 13.549                  & 12.887                  & 12.363                  & 11.933                  & 10.248                  & 8.617                   & 7.608                   & 6.591                   \\
$T=10$                  & 15.049                  & 14.701                  & 14.254                  & 13.616                  & 13.219                  & 11.675                  & 10.309                  & 8.830                   & 7.571                   \\ \hline
\end{tabular}
\caption{MSEs in the fOU process cases with different $H, \ T$ for $s=5$, $N=2^{12}$.}
\label{table4}
\end{table}

In Figure \ref{figure3}, we take the MSEs of the theoretically optimal predictor (\ref{predictfou}) as the benchmarks, which are the smallest MSEs we can theoretically obtain. By Proposotion 6.3 in \cite{Hu2005IntegralTA},
\begin{displaymath}
    \begin{aligned}
        &\operatorname{Cov}\left[\int_0^T \mathrm{e}^{-\frac{T-s}{2}} \mathrm{d} B_s^H, \int_0^{t_i} \mathrm{e}^{-\frac{t_i-s}{2}} \mathrm{d} B_s^H\right] \\
        =&\mathbb{E}\left[\int_0^T \mathrm{e}^{-\frac{T-s}{2}} \mathrm{d} B_s^H\int_0^{t_i} \mathrm{e}^{-\frac{t_i-s}{2}} \mathrm{d} B_s^H\right] \\
        =&\mathbb{E}\left[\int_0^T \mathrm{e}^{-\frac{T-s}{2}} \mathrm{d} B_s^H\int_0^{T} \mathrm{e}^{-\frac{t_i-s}{2}}\mathbf{1} _{\left [0,t_i\right ]}\left (s\right ) \mathrm{d} B_s^H\right] \\ 
        =&\int_0^T \Gamma_{H, T}^* \mathrm{e}^{-\frac{T-s}{2}} \Gamma_{H, T}^* \mathrm{e}^{-\frac{t_i-s}{2}}\mathbf{1} _{\left [0,t_i\right ]}\left (s\right ) \mathrm{d} s,
    \end{aligned}
\end{displaymath}
where
\begin{displaymath}
    \Gamma_{H, T}^* f(s)=\left(H-\frac{1}{2}\right) \kappa_H s^{\frac{1}{2}-H} \int_s^T u^{H-\frac{1}{2}}(u-s)^{H-\frac{3}{2}} f(u) \mathrm{d}u.
\end{displaymath}
Similarly, we can do the same for $\operatorname{Cov}\left[\int_0^{t_i} \mathrm{e}^{\int_s^{t_i} \mu (s)\mathrm{d} s} \mathrm{d} B_s^H, \int_0^{t_j} \mathrm{e}^{\int_s^{t_j} \mu (s)\mathrm{d} s} \mathrm{d} B_s^H\right]$, and consequently get each item in (\ref{predictfou}).

In Figure \ref{figure3}, we compare the MSEs of the theoretically optimal predictor (\ref{predictfou}) with the MSEs of the trained deep neural network. The MSE gap is confined within a small range near 0, which demonstrates that the predictions obtained through our proposed framework are not significantly different from the theoretical optimal predictions.

\begin{figure}[htbp]
\centering
\includegraphics[width=0.95\textwidth,trim=40 0 40 0,clip]{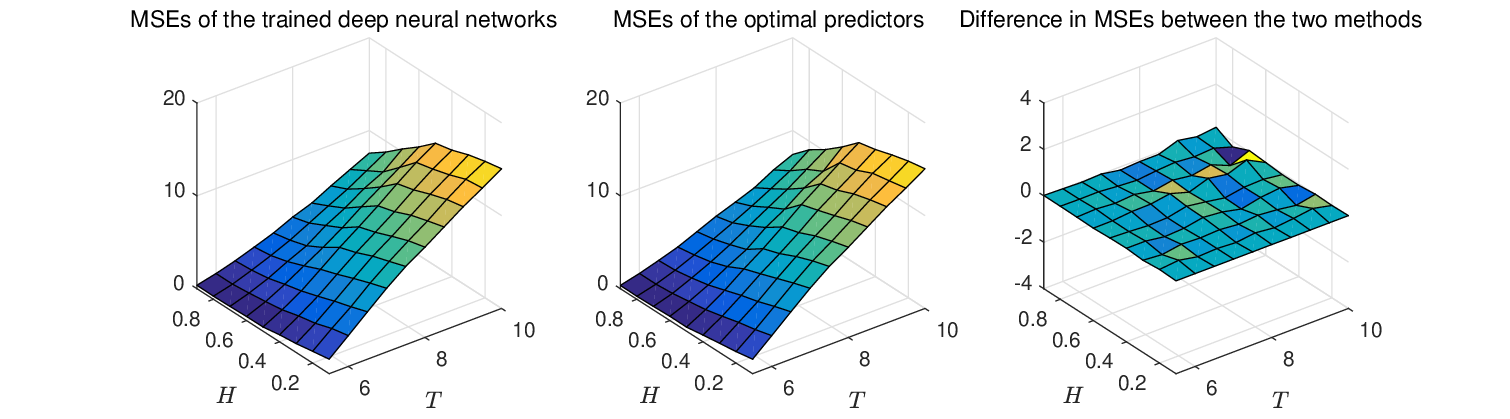}
\caption{MSEs of the two methods in the fOU process cases for $s=5$, $N=2^{12}$ and difference.}
\label{figure3}
\end{figure}

Consider the predictions of the fOU processes based on continuous observations, i.e.,
\begin{displaymath}
    \mathbb{E}\left[A_{T} \mid \mathcal{F} ^{C,A}_s\right].
\end{displaymath}
where $A_{t}$ is the solution of the SDE (\ref{FOU}) with $k(\cdot)=a_0=0, \ a(\cdot)=\frac{1}{2}$, and $\sigma(\cdot)=1$. This prediction problem is well solved by \cite{Fink2013}. The solution $A_t$ is given by
\begin{displaymath}
    A_t=\int_0^t e^{-\frac{t-v}{2}} \mathrm{d} B^H_v,
\end{displaymath}
and
\begin{displaymath}
    \mathbb{E}\left[A_{T} \mid \mathcal{F} ^{C,A}_s\right]=A_s e^{-\frac{T-s}{2}}+\int_0^s \Psi_c^H(s, T, v) \mathrm{d} B^H_v,
\end{displaymath}
where for $v \in(0, s)$,
\begin{displaymath}
\Psi_c^H(s, T, v) =\frac{\sin\left(\left(H- \frac{1}{2}\right)\pi  \right)}{\pi} v^{-H+ \frac{1}{2}}(s-v)^{-H+ \frac{1}{2}} \int_s^T \frac{z^{H- \frac{1}{2}}(z-s)^{H- \frac{1}{2}}}{z-v} e^{-\frac{T-v}{2}} \mathrm{d} z,
\end{displaymath}
and for $v \in\{0, s\}$, we have $\Psi_c^H(s, T, v)=0$.

Let $s=5$ and $T=10$. Given sample paths, Figure \ref{figure4} shows that the predictions based on discrete observations obtained through the deep neural network method converge to predictions based on continuous observations. When $H=0.1,0.3,0.7$, convergence is faster.

\begin{figure}[htbp]
\centering
\includegraphics[width=0.90\textwidth,trim=20 0 0 0,clip]{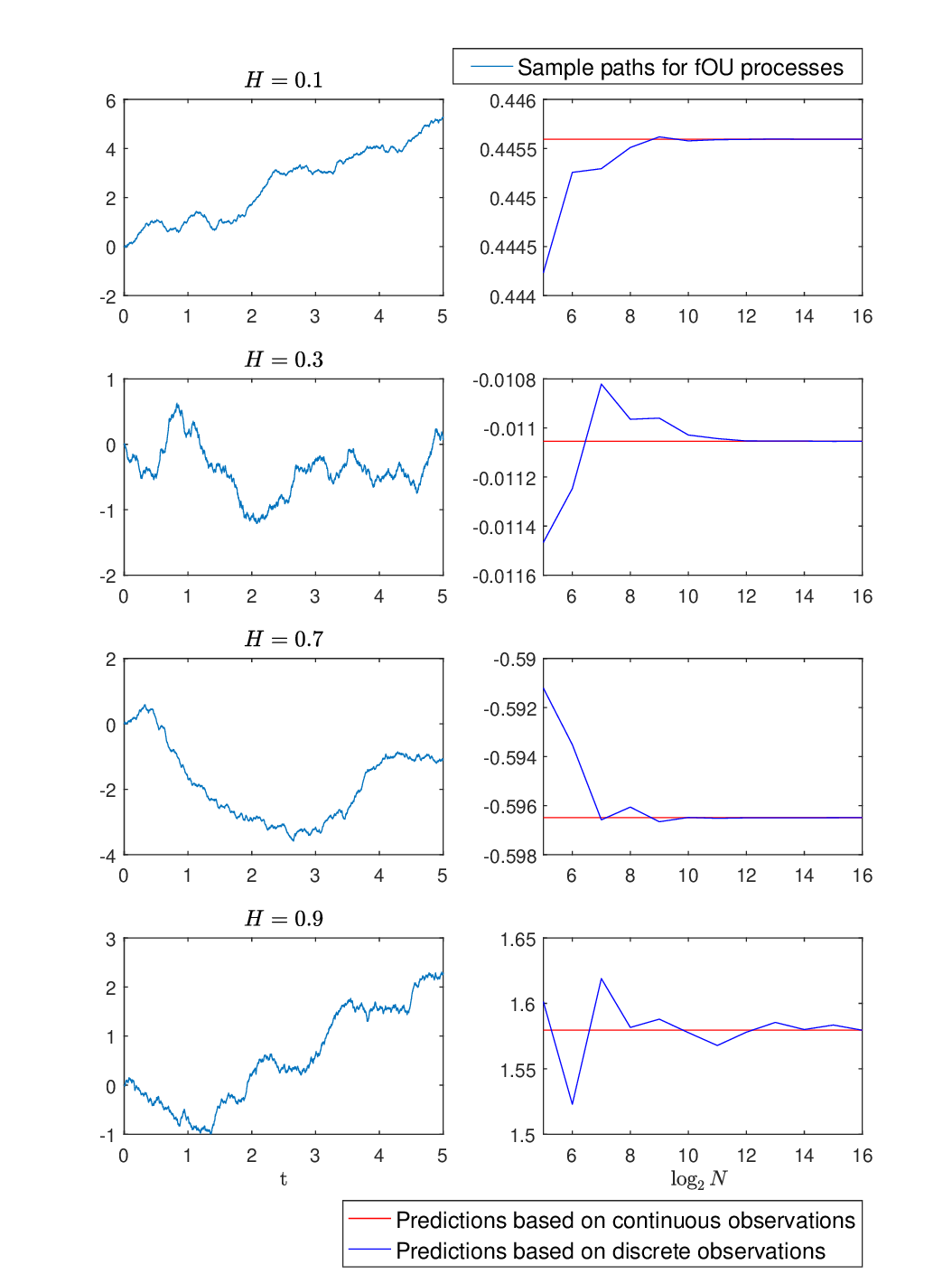}
\caption{Sample paths, predictions based on discrete observations and continuous observations in the fOU process cases for $s=5$, $T=10$.}
\label{figure4}
\end{figure}

\section{Conclusions}
\label{sec:5}
In this paper, we propose a deep learning method to predict path-dependent processes by discretely observing historical information. This method is implemented by considering the prediction as a nonparametric regression and obtaining the regression function through simulated samples and deep neural networks. We theoretically demonstrate the applicability of the method to fBm and the solutions of some stochastic differential equations driven by it, and further discuss the scope of the method. Our proposed framework for predicting path-dependent processes $W_t$ is broadly applicable to situations where the relationship between discrete observations of $V_t$ and $W_t$ can be described by a series of $\left(p, C\right)$-smooth functions. By dividing several input regions, we can also consider the input-output relationship case by case. Moreover, with the frequency of discrete observations tending to infinity, the predictions based on discrete observations converge to the predictions based on continuous observations. 

In numerical simulation, we apply the method to the fBm and the fOU process as examples. The further into the future one needs to predict, the more difficult to predict accurately. It is not the case that the larger or smaller the $H$, the smaller the prediction error is. The forecast time horizon affects the role of the Hurst index on forecast accuracy. Moreover, it is worth noting that the not the more historical information is available from discrete observations, the smaller the prediction bias. Comparing the results with the theoretical optimal predictions and taking the MSE as a measure, the numerical simulations demonstrate that the method can generate accurate results. Comparing the predictions based on discrete observations and continuous observations, the predictions based on discrete observations obtained through the deep neural network method converge to predictions based on continuous observations, which implies that we can make approximations by the method.

\acks{The authors are grateful for financial support from the National Key R\&D Program of China (No. 2023YFA1009200).}

\newpage

\appendix
\section{}
\label{app:theorem}
Based on the notion of piecewise smooth functions proposed by \cite{Imaizumi2019} and boundary fragment classes developed by \cite{Dudley1974}, we introduce a special class of functions that take 1 on some regions of the state space whose boundaries are formed by a series of smooth functions, and demonstrate the existence of neural networks that can approximate this class of functions.
 
\begin{definition}\label{dfn1}
For  $\lambda, J, R \in \mathbb{N}$ and $\mathcal{P} \subseteq[1, \infty) \times \mathbb{N}$, we define
\begin{displaymath}
    \mathcal{G}\left(R, \mathcal{P}\right):=\left\{f\left(x\right)=\prod_{i=1}^R f_i\left(x\right): f_i\left(x\right) \in \mathcal{H} \left(p_i, D_i\right) \text { and }\left(p_i, D_i\right) \in \mathcal{P}\right\} ,
\end{displaymath}
\begin{displaymath}
    \mathcal{K}\left(\lambda, J, \mathcal{P}\right):=\left\{f\left(x\right)=\max _{1 \leq j \leq J} f_j(x): f_j(x) \in \mathcal{G}\left(R_j, \mathcal{P}_j\right), \lambda=\sum_{j=1}^J R_j, \mathcal{P}_j \subseteq \mathcal{P}\right\} .
\end{displaymath}
\end{definition}

Obviously, the functions in the $\mathcal{K}\left(\lambda, J, \mathcal{P}\right)$ take 1 for some blocks and 0 for the rest of the regions and the functions themselves are not differentiable or even discontinuous at the edges of these regions. It is worth noting that each edge function can have a different smoothness $p=q+s$ and a different input dimension $D$.

\begin{lemma}
\label{lemma4}
Let the training sets  $(X, Y),\left(X_1, Y_1\right), \ldots,\left(X_n, Y_n\right)$ be independent and identically distributed random values such that $\operatorname{supp}(X)$ is bounded and $\mathbb{E}\left[e^{c Y^2} \right]<\infty$
for some constant $c>0$. Let $f\in \mathcal{K}(\lambda, J, \mathcal{P})$, where $\lambda, J \in \mathbb{N}$, $\mathcal{P} \subseteq[1, \infty) \times \mathbb{N}$, and $Q_i \in \mathbb{N}, \ i=1,\ldots,\lambda$, sufficiently large. Then, there exist $D_{i }, \ p_{i}, \ q_{i}, \ R_{i}, \ i=1,\ldots,\lambda,$ and a neural network $f^\theta \in \mathcal{Y} \left(L, r\right)$ with the property that
\begin{displaymath}
    \|f(x)-f^\theta\left(x\right)\|_{2,\left[-c_1,c_1\right]^d} \leq c_{2} c_{1}^{4\left(\max_{0 \leq i \leq \lambda}q_{i }+1\right)} \max_{0 \leq i \leq \lambda} Q_i^{-2 p_i},
\end{displaymath}
where $c_1$ and $c_2$ are constants,
\begin{displaymath}
    \begin{aligned}
L  =&\max_{1 \leq i \leq  \lambda} \left[5 Q_i^{D_i}+\left\lceil\log _4\left(Q_i^{2 p_i+4{D_i} \left(q_i+1\right)}  e^{4\left(q_i+1\right)\left(Q_i^{D_i}-1\right)}\right)\right\rceil \right.\\
&\left.\left\lceil\log _2(\max_{1 \leq i \leq  \lambda} \{q_i, {D_i}\}+1)\right\rceil+\left\lceil\log _4\left(Q_i^{2 p_i}\right)\right\rceil\right]+4+\left\lceil\log _2 d\right\rceil, \\
r  =&\sum_{i=1}^\lambda 132 \cdot 2^{D_i} \left\lceil e^{D_i}\right\rceil \left(\begin{array}{c}{D_i}+q_i \\ {D_i}\end{array}\right)  \max_{1 \leq i \leq  \lambda} \left\{q_i+1, {D_i}^2\right\}.
\end{aligned}
\end{displaymath}
\end{lemma}
\begin{proof}
By the definition of $\mathcal{K}\left(\lambda, J, \mathcal{P}\right)$, for arbitrary $f\left( x \right) \in \mathcal{K}\left(\lambda, J, \mathcal{P}\right)$, there exists a set of $\left(p, C\right)$-smooth functions $f_i\left( x \right), i=1,\ldots,$ $\sum_{j=1}^J R_j$, such that
\begin{displaymath}
    \phi_{1, i}=f_i\left( x \right),\ \phi_{2, i}=\mathbf{1}_{\left(0, \infty\right)}\left(\phi_{1, i}\right), \ \phi_{3, j}=\prod_{i=\sum_{m=1}^{j-1} R_m+1}^{\sum_{m=1}^j R_m} \phi_{2, i}, 
\end{displaymath}
\begin{displaymath}
    \phi_4=\sum_{j=1}^J \phi_{3, j}, \ f\left( x \right)=\mathbf{1}_{(0, \infty)}\left( \phi_4 \right).
\end{displaymath}
Thus, there exists a set of functions $f_i\left( x \right)$, $i=1,\ldots,$ $\sum_{j=1}^J R_j,$ which is $\left( p_i,C \right)$-smooth respectively, $p_i=q_i+s_i$ for some $q_i \in \mathbb{N}_0$ and $0<s_i \leq 1$. Based on Theorem 2 introduced by \cite{Kohler2021}, there exists $\hat{\phi}_{1, i}\in \mathcal{Y} \left(L_i, r_i\right)$ such that
\begin{displaymath}
    \left\|f_i(x)-\hat{\phi}_{1, i}\left(x\right)\right\|_{\infty,\left[-c_1,c_1\right]^d} \leq c^{i}_{3} {c_1}^{4\left(q_{i}+1\right)} Q_i^{-2 p_i},
\end{displaymath}
where $c_3^i$ is a constant,
\begin{displaymath}
    \begin{aligned}
L_i  =&5 Q_i^{D_i}+\left\lceil\log _4\left(Q_i^{2 p_i+4{D_i} \left(q_i+1\right)}  e^{4\left(q_i+1\right)\left(Q_i^{D_i}-1\right)}\right)\right\rceil \left\lceil\log _2(\max_{1 \leq i \leq  \lambda} \{q_i, {D_i}\}+1)\right\rceil\\
&+\left\lceil\log _4\left(Q_i^{2 p_i}\right)\right\rceil, \\
r_i  =& 132 \cdot 2^{D_i} \left\lceil e^{D_i}\right\rceil \left(\begin{array}{c}{D_i}+q_i \\ {D_i}\end{array}\right)  \max_{1 \leq i \leq  \lambda} \left\{q_i+1, {D_i}^2\right\}.
\end{aligned}
\end{displaymath}

Without loss of generality, for $1 \leq u \leq D_{i}$, let $ x_{i_1}=x_{u}$, \ $g_i=f_i+x_1$ and $\phi_{1, i}^{\prime}=\hat{\phi}_{1, i}+x_{1}$,
\begin{displaymath}
\begin{aligned}
& \left\|\mathbf{1}_{\left\{f_i>0\right\}}-\mathbf{1}_{\left\{\hat{\phi}_{1, i}>0\right\}}\right\|_{2,\left[-c_1,c_1\right]^d}^2 \\
 =&\int\left(\mathbf{1}_{\left\{f_i>0\right\}}-\mathbf{1}_{\left\{\hat{\phi}_{1, i}>0\right\}}\right)^2 \mathrm{d}x \\
 =&\int_{-c_1}^{c_1}\ldots\int_{-c_1}^{c_1} \mathbf{1}_{\left\{f_i>0,\ \hat{\phi}_{1, i}\leq 0\right\}}+\mathbf{1}_{\left\{f_i\leq0,\ \hat{\phi}_{1, i}>0\right\}}\mathrm{d} x_{1}\ldots \mathrm{d} x_{D_i} . \\
\end{aligned}
\end{displaymath}
For fixed $\left(x_{2}, \ldots, x_{D_i}\right) \in\left[-c_1, c_1\right]^{D_i-1}$,  $\mathbf{1}_{\left\{f_i>0,\ \hat{\phi}_{1, i}<0\right\}}=\mathbf{1}_{\left[\phi_{1, t}^{\prime},\ g_i\right)}$. Thus,
\begin{displaymath}
    \int_{-c_1}^{c_1} \mathbf{1}_{\left\{f_i>0,\ \phi_{1, i} \leq 0\right\}} \mathrm{d} x_{1} \leq\left(g_i-\phi_{1, i}^{\prime}\right) \vee 0 .
\end{displaymath}
Similarly,
\begin{displaymath}
    \int_{-c_1}^{c_1} \mathbf{1}_{\left\{f_i \leq 0,\ \phi_{1, i}>0\right\}} \mathrm{d} x_{1} \leq\left(\phi_{1, i}^{\prime}-g_i\right) \vee 0 .
\end{displaymath}
Notice that $\left(b \vee 0\right)+\left(-b \vee 0\right)=|b|$, for any $x \in\left[-c_1, c_1\right]^d$, we have
\begin{displaymath}
    \begin{aligned}
& \int_{-c_1}^{c_1}\ldots\int_{-c_1}^{c_1} \mathbf{1}_{\left\{f_i>0,\ \hat{\phi}_{1, i} \leq 0\right\}}+\mathbf{1}_{\left\{f_i \leq 0,\ \hat{\phi}_{1, i}>0\right\}} \mathrm{d} x_{1}  \ldots \mathrm{d} x_{D_i} \\
 \leq & \int_{-c_1}^{c_1}\ldots\int_{-c_1}^{c_1}\left(\left(g_i-\phi_{1, i}^{\prime}\right) \vee 0\right)+\left(\left(\phi_{1, i}^{\prime}-g_i\right) \vee 0\right) \mathrm{d} x_{2} \ldots \mathrm{d} x_{D_i} \\
 \leq & c_{4}^i {c_1}^{4\left(q_i+1\right)} Q_i^{-2 p_i}
\end{aligned}
\end{displaymath}
for a constant $c_4^i>0$.

Based on Lemma 4 and Lemma 20 in \cite{Kohler2021}, there exists a network $f_{sq}$  with the network architecture $\left(1,18\right)$ satisfying $f_{sq}\left(m, n\right)=|m+n|-|m-n|$, for $m,n \in \left[0,1\right]$. 
Let 
\begin{displaymath}
    w=\left\lceil\log_2 d\right\rceil, \ \left(z_1, \ldots, z_{2^w}\right)=\left(x_{1}, x_{2}, \ldots, x_{d}, 1, \ldots, 1\right).
\end{displaymath}
In the first layer of $f_{\text{mult}}$, we compute 
\begin{displaymath}
    f_{sq}\left(z_1, z_2\right), \ f_{sq}\left(z_3, z_4\right), \ldots, f_{ q}\left(z_{2 w-1}, \ z_{2 q}\right),
\end{displaymath}
 which can be done by $18 \cdot 2^{w-1} \leq 18 d$ neurons. The output of the first layer is a vector of length $2^{w-1}$. This process is repeated for the output vectors until the output is a one-dimensional vector. If $m=0$, then we have $f_{sq}\left(m, n\right)=|n|-|n|=0$. Based on mathematical induction, if a element of the vector $x$ is equal to 0, $f_{\text {mult }}(x)=0$. Therefore, for $x \in \left[0,1\right]^d$, there exists a neural network $f_{\text{mult}}$ with the network architecture $\left(\left\lceil\log _2 d\right\rceil, 18 d\right)$ such that
\begin{equation}
\label{mult}
    f_{\text {mult }}(x) \begin{cases}>0, & \text { if } x_{i}>0, \ \forall \ 0 \leq i \leq d, \\ =0, & \text { otherwise. }\end{cases}    
\end{equation}

Let $f_{\text{demo}}\left(x\right)=-\sigma\left(-c_2\sigma\left(x\right)+1\right)+1$, where $\sigma\left(x\right)$ is the ReLU activation function. Then, for $x \in \left[0,1\right]$, there exists a neural network $f_{\text{demo}} \in \mathcal{Y} \left(3,2\right)$ such that
\begin{equation}
\label{demo}
    f_{\text{demo}}\left(x\right)=\mathbf{1}_{\left(\frac{1}{c_5}, \infty\right)}\left(x\right),    
\end{equation}
where $c_5 \geq 1$ is a constant.

Let
\begin{displaymath}
    f^\theta\left( x \right)=\mathbf{1}_{(0, \infty)} \left( \sum_{j=1}^J \left(  \prod_{i=\sum_{i=1}^{j-1} R_i+1}^{\sum_{i=1}^j R_i} \left( \mathbf{1}_{(0, \infty)} \left(\hat{\phi}_{1, i}\left(x\right) \right)\right)\right) \right).
\end{displaymath}
By (\ref{mult}) and (\ref{demo}), given a sufficiently small $c_2$, $\prod\left(\mathbf{1}_{(0, \infty)}\left ( x  \right ) \right)$ and $\mathbf{1}_{(0, \infty)}\left ( x  \right )$ can be implemented through networks $f_{\text {mult }}(x)$ and $f_{\text{demo}}\left(x\right)$. Therefore, $f^\theta \in \mathcal{Y} \left(L, r\right)$, where
\begin{displaymath}
    \begin{aligned}
L =& \max_{1 \leq i \leq  \lambda} \left[5 Q_i^{D_i}+\left\lceil\log _4\left(Q_i^{2 p_i+4{D_i} \left(q_i+1\right)}  e^{4\left(q_i+1\right)\left(Q_i^{D_i}-1\right)}\right)\right\rceil \right.\\
&\left. \left\lceil\log _2(\max_{1 \leq i \leq  \lambda} \{q_i, {D_i}\}+1)\right\rceil+\left\lceil\log _4\left(Q_i^{2 p_i}\right)\right\rceil\right]+4+\left\lceil\log _2 d\right\rceil, \\
r  =&\sum_{i=1}^\lambda 132 \cdot 2^{D_i} \left\lceil e^{D_i}\right\rceil \left(\begin{array}{c}{D_i}+q_i \\ {D_i}\end{array}\right) \max_{1 \leq i \leq  \lambda}\left\{q_i+1, {D_i}^2\right\}.
\end{aligned}
\end{displaymath}
Moreover,
\begin{displaymath}
    \begin{aligned}
&  \|f(x)-f^\theta\left(x\right)\|_{2,\left[-c_1,c_1\right]^d} \\
 \leq & \sum_{j=1}^J\left\|\prod_{i=\sum_{i=1}^{j-1} R_i+1}^{\sum_{i=1}^j R_i} I_{\left\{f_i>0\right\}}(x)-\prod_{i=\sum_{i=1}^{j=-1} R_i+1}^{\sum_{i=1}^j R_i} I_{\left\{\hat{\phi}_{1, i}>0\right\}}(x)\right\|_{2,\left[-c_1,c_1\right]^d}\\
 \leq & c_{5} \max_{1 \leq i \leq \lambda}\left\|f_i(x)-\hat{\phi}_{1, i}\left(x \right)\right\|_{2,\left[-c_1,c_1\right]^d} \\
 \leq & c_{6} c_{1}^{4\left(\max_{1 \leq i \leq \lambda}q_{i}+1\right)} \max_{1 \leq i \leq \lambda} Q_i^{-2 p_i}, 
\end{aligned}
\end{displaymath}
where $c_5,c_6$ are constants.
\end{proof}

\begin{remark}
  Lemma \ref{lemma4} illustrates that a sufficiently deep neural network can approximate the function very well. Based on Theorem 2 introduced by \cite{Kohler2021}, a sufficiently wide neural network can also yield similar results. More specifically, replacing $L$ and $r$ with 
\begin{displaymath}
    \begin{aligned}
L  =&9+\left\lceil\log _4\left(\max_{1 \leq i \leq \lambda} Q_i^{2 p_i}\right)\right\rceil\left(\left\lceil\log _2\left(\max_{1 \leq i \leq \lambda}\left\{D_{i }, q_{i }\right\}+1\right)\right\rceil+1\right)\\
&+\left\lceil\log _2 \max_{1 \leq j \leq J}R_{j }\right\rceil, \\
r  =&\sum_{i=1}^\lambda 2^{D_i} 64\left(\begin{array}{c}
D_i+q_i \\
D_i
\end{array}\right) D_i^2\left(q_i+1\right) Q_i^{D_i} .
\end{aligned}
\end{displaymath}
The Lemma \ref{lemma4} still holds.  
\end{remark} 

\newpage
\bibliography{sample}

\end{document}